\documentclass[nohyperref]{article}

\usepackage{microtype}
\usepackage{graphicx}
\usepackage{subfigure}
\usepackage{booktabs} %

\usepackage{hyperref}

\usepackage[accepted]{icml2022}

\usepackage{amsmath}
\usepackage{amssymb}
\usepackage{mathtools}
\usepackage{amsthm}

\csname icml\endcsname
\csname includeappendix\endcsname

\usepackage[utf8]{inputenc} %
\usepackage[T1]{fontenc}    %
\usepackage{hyperref}       %
\usepackage{url}            %
\usepackage{booktabs}       %
\usepackage{amsfonts}       %
\usepackage{nicefrac}       %
\usepackage{microtype}      %
\usepackage{xcolor}         %

\usepackage[T1]{fontenc}
\usepackage{enumitem}
\usepackage{color, colortbl}
\usepackage{amsmath, amssymb, amsthm, mathrsfs}
\usepackage{mathtools, mathabx}
\usepackage{color}
\usepackage{cancel}
\usepackage{wrapfig}
\usepackage{caption}
\usepackage{xfrac}
\usepackage{algorithm}
\usepackage{algorithmic}
\usepackage{mdframed}
\usepackage{hyperref}
\usepackage{dsfont}
\usepackage{xcolor}
\usepackage{inputenc}
\usepackage{times}
\usepackage{wrapfig}

\usepackage[normalem]{ulem}
\hypersetup{
	colorlinks,
	linkcolor={red!40!gray},
	citecolor={blue!40!gray},
	urlcolor={blue!70!gray}
}

\usepackage{cleveref}

\usepackage{wrapfig}

\definecolor{Gray}{gray}{0.9}
\definecolor{GrayH}{gray}{0.7}

\newcommand{\ind}{\mathds{1}}

\newcommand{\dd}{\mathrm{d}}

\newcommand{\diam}{\mathrm{diam}}

\newtheorem{theorem}{Theorem}
\newtheorem{lemma}{Lemma}

\newtheorem*{theorem*}{Theorem}
\newtheorem{corollary}{Corollary}
\newtheorem*{claim*}{Claim}
\newtheorem*{fact*}{Fact}
\newtheorem*{observation*}{Observation}
\theoremstyle{definition}
\newtheorem{definition}{Definition}
\newtheorem{remark}{Remark}
\newtheorem*{remark*}{Remark}

\newtheorem*{example*}{Example}
\DeclareMathOperator*{\argmin}{arg\,min}

\newcommand{\iidsim}{\overset{\text{iid}}{\sim}}

\newcommand{\michael}[1]{{\leavevmode\color{red}{Michael:\ #1}}}

\global\long\def\dimh{\mathrm{dim}_{\mathrm{H}}}%
\global\long\def\diam{\mathrm{diam}}

\usepackage{stmaryrd}

\icmltitlerunning{Generalization Bounds using Lower Tail Exponents in Stochastic Optimizers}

\begin{document}

\twocolumn[
\icmltitle{Generalization Bounds using Lower Tail Exponents in Stochastic Optimizers}

\begin{icmlauthorlist}
\icmlauthor{Liam Hodgkinson}{icsiberk}
\icmlauthor{Umut \c{S}im\c{s}ekli}{inria}
\icmlauthor{Rajiv Khanna}{purdue}
\icmlauthor{Michael W. Mahoney}{icsiberk}
\end{icmlauthorlist}

\icmlaffiliation{icsiberk}{ICSI and Department of Statistics, University of California, Berkeley, USA}
\icmlaffiliation{inria}{INRIA --- D\'{e}partement d'Informatique de l'\'{E}cole Normale Sup\'{e}rieure, PSL Research University, Paris, France}
\icmlaffiliation{purdue}{Department of Computer Science, Purdue University, Indiana, USA}

\icmlcorrespondingauthor{Liam Hodgkinson}{liam.hodgkinson@gmail.com}

\icmlkeywords{Machine Learning, ICML}

\vskip 0.3in
]

\printAffiliationsAndNotice{}  %

\begin{abstract}

Despite the ubiquitous use of stochastic optimization algorithms in machine learning, the precise impact of these algorithms and their dynamics on generalization performance in realistic non-convex settings is still poorly understood. While recent work has revealed connections between generalization and heavy-tailed behavior in stochastic optimization, this work mainly relied on continuous-time approximations; and a rigorous treatment for the original discrete-time iterations is yet to be performed. To bridge this gap, we present novel bounds linking generalization to the \emph{lower tail exponent} of the transition kernel associated with the optimizer around a local minimum, in \emph{both} discrete- and continuous-time settings. To achieve this, we first prove a data- and algorithm-dependent generalization bound in terms of the celebrated Fernique--Talagrand functional applied to the trajectory of the optimizer. Then, we specialize this result by exploiting the Markovian structure of stochastic optimizers, and derive bounds in terms of their (data-dependent) transition kernels. We support our theory with empirical results from a variety of neural networks, showing correlations between generalization error and lower tail exponents.

\end{abstract}

\section{Introduction}

Fundamental to the operation of modern machine learning is \emph{stochastic optimization}: the process of minimizing an objective function via the simulation of random elements. 
Its practical utility is matched by its theoretical depth; for decades, optimization theorists have sought to explain the surprising generalization ability of stochastic gradient descent (SGD) and its various extensions for non-convex problems --- most recently in the context of neural networks and deep learning. 
Classical convex optimization-centric approaches fail to explain this phenomenon.

\begin{figure}[t]
    \centering
    \includegraphics[width=0.23\textwidth]{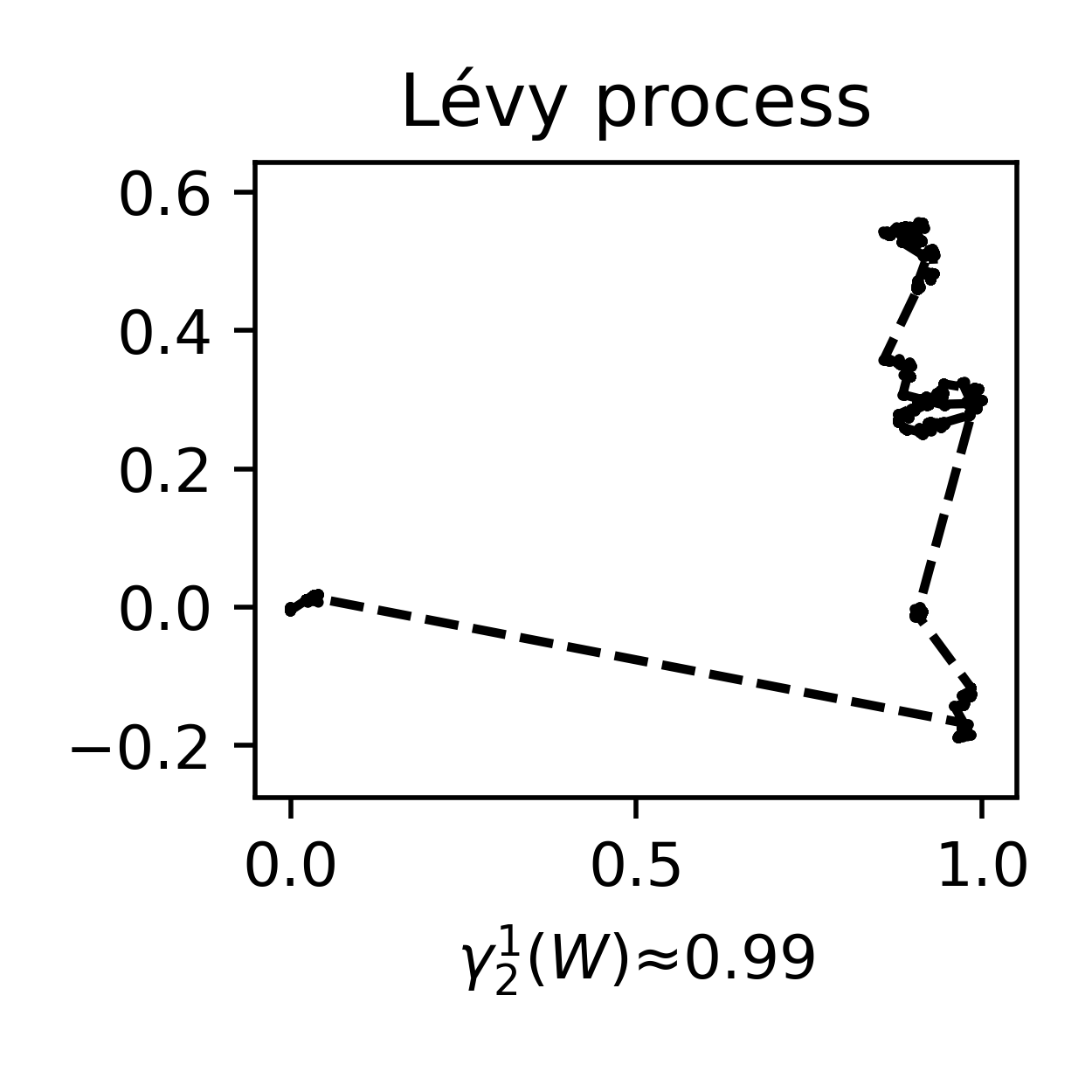}
    \includegraphics[width=0.23\textwidth]{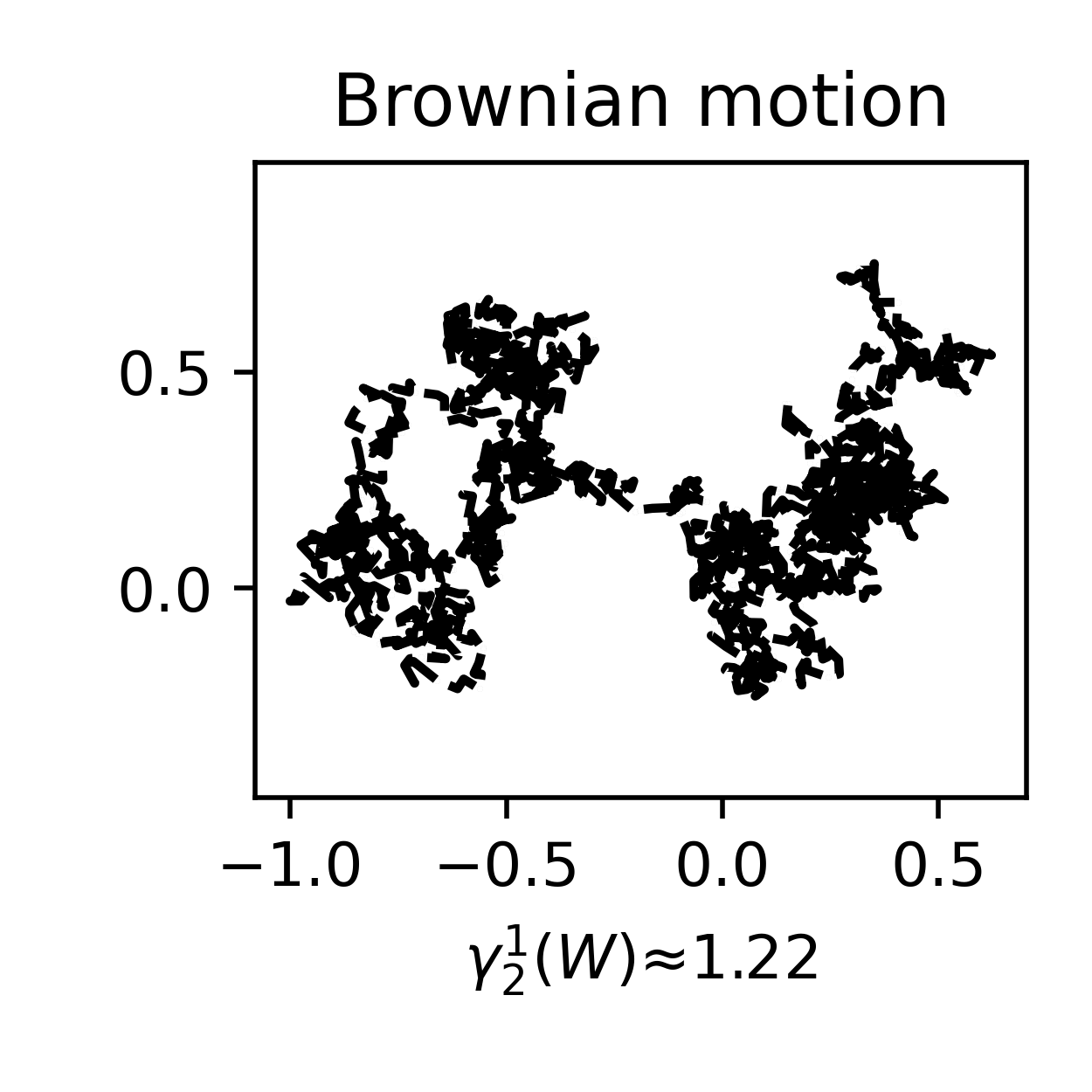}
    \vspace{-.5cm}
    \caption{Discrete sample path approximations of a heavy-tailed $\alpha$-stable L\'{e}vy process ($\alpha = 1.5$), and standard Brownian motion. Estimates of our normalized Fernique--Talagrand functional $\gamma_2^{1}\left(\cdot\right)$ is reported under each figure (see Section~\ref{sec:ftnorm}). Observe this functional is reduced with smaller tail index and
    ``tighter clustering" of the trajectory.    }
    \label{fig:BrownLevy}
    \vspace{-.5cm}
\end{figure}

There has been an increasing number of attempts for developing generalization bounds for non-convex learning settings.  
This work has approached the problem from different perspectives, such as information theory, compression/sparsity/intrinsic dimension, or implicit (algorithmic) regularization (details to be provided in Section~\ref{sec:relwork}). 
Among these approaches, a promising direction has been to consider \emph{optimization trajectories}, rather than single point estimates obtained during (or at the end of) the optimization process (e.g., \citet{neyshabur2017exploring,xu2017information,arora2018stronger}).  This addresses a plausible concern that single points may not necessarily be able to capture all the information regarding generalization.
There has also been significant empirical evidence (using a wide variety of approaches) supporting this idea \cite{jastrzkebski2017three,xing2018walk,martin2018implicit_JRNL,martin2020predicting_JRNL,Jastrzebski2020,jastrzebski2021catastrophic}.
Of particular interest to us are the recent empirical developments linking \emph{heavy-tailed fluctuations} in  optimization trajectories to generalization performance \citep{simsekli2019tail,gurbuzbalaban2020heavy,hodgkinson2020multiplicative}.

The heavy-tailed dynamics observed in SGD exhibit qualitatively different behavior from Gaussian dynamics (see Figure \ref{fig:BrownLevy}), and it typically coincides with improved performance \cite{martin2019traditional}. As a first step in providing theoretical justification for these observations, \citet{simsekli2020hausdorff} used fractal dimension theory to prove a generalization bound involving the tail exponent of the iterates obtained from a stochastic optimizer. While they brought a new perspective, their analysis unfortunately assumes a continuous-time Feller process model as an approximation for the optimizer trajectories, and it is unclear how their techniques can be extended beyond this setting. Of course, optimization procedures used in machine learning are not continuous-time, and a rigorous treatment for the original discrete-time setting is still missing.

In this study, we address this issue and present a new mathematical framework that is sufficiently flexible to treat the discrete-time setting and to recover (and improve) the results of \citet{simsekli2020hausdorff} in the continuous-time setting.
Similar to \citet{simsekli2020hausdorff}, our underlying strategy is to explicitly include \emph{all} the information surrounding the dynamics of the optimizer near the optimum as it applies to generalization performance.
  Therefore, the quantity of interest in this work is an \emph{accumulated generalization gap} over the trajectory of the optimizer: 
\begin{equation}
\label{eq:AccError}
\sup_{t \in [t_1,t_2]} |\mathcal{E}_n(W_t)| = \sup_{w \in \{W_t\}_{t=t_1}^{t_2}} |\mathcal{E}_n(w)|,
\end{equation}
where $\{W_t\}_{t=t_1}^{t_2}$ denotes the optimizer trajectory from `time' $t_1$ to $t_2$,  $\mathcal{E}_n(w) := \mathcal{R}_n(w) - \mathcal{R}(w)$ is the generalization gap (or the generalization error), and $\mathcal{R}$ and $\mathcal{R}_n$ denote the population and empirical risk functions, respectively. As the generalization error towards the end of training is of greatest interest, $t_1$ will typically be chosen so that $\{W_t\}_{t=t_1}^{t_2}$ lies in the domain of attraction of a local optimum. %
The advantage of considering~\eqref{eq:AccError} is that the presence of the supremum enables tools from analytic probability theory surrounding uniform error bounds. Therefore, to bound and estimate (\ref{eq:AccError}), we will draw from this literature --- in particular, a certain functional of  \citet{fernique1971regularite} and  \citet{talagrand1996majorizing}.%

\subsection{Contributions}

Recalling that any practical stochastic optimization algorithm can be written as a Markov process \cite{hodgkinson2020multiplicative}, our main contribution is encapsulated in the following informal theorem, with the precise statement resulting from combining Theorem \ref{thm:NormFL} and Corollary \ref{cor:BulkExponent} (see Section~\ref{sec:main}). 

\begin{theorem*}[Informal]
Assume that an optimizer satisfies the following in the neighborhood of a local minimum $w^\ast$ : there is a \emph{lower tail exponent} $\alpha$ such that for any $k$,
\[\mathbb{P}(\|W_{k+1}-w^\ast\|\leq r \,\vert\, W_k = w^\ast) \approx \mathcal{O}(r^\alpha),\mbox{ as } r\to 0^+,\]
where $W_k$ is the $k^{\text{th}}$ iterate of the optimizer. Then, the expected accumulated generalization gap (\ref{eq:AccError}) in the neighborhood of $w^\ast$ is approximately bounded by the sum of $C\sqrt{\alpha / n}$ and the mutual information between the data and the trajectory of the weights, where $C > 0$ is a constant.%
\end{theorem*}

To prove this result, we develop a theoretical framework for investigating the generalization properties of stochastic optimizers, in two parts.
In the first part, \textbf{(I)}, the generalization gap attributable to optimizer dynamics is effectively reduced to a normalized \emph{Fernique--Talagrand} (FT) functional.
This functional is introduced in Section \ref{sec:FT}; and in Theorem \ref{thm:NormFL}, we obtain a sharp (up to constants) bound on the accumulated generalization gap (\ref{eq:AccError}) in terms of our normalized FT functional applied to the trajectory of the optimizer, and the mutual information. 
In the second part, \textbf{(II)}, we proceed to bound the normalized FT functional in two ways: %
\begin{enumerate}[leftmargin=*]
    \item \emph{Hausdorff dimension: } By considering the trajectory of a continuous-time stochastic optimization model and  conducting a similar fractal dimension analysis to \citet{simsekli2020hausdorff}, we recover and sharpen their results in Corollary \ref{cor:Hausdorff}. In particular, using our framework, we remove many nontransparent assumptions while improving convergence rates in generalization for heavy-tailed continuous-time processes from $\mathcal{O}(n^{-1/2} \log n)$ to $\mathcal{O}(n^{-1/2})$. 
    \item \emph{Transition kernel: } In Theorem \ref{thm:bound_transition_kernel}, we bound the expected normalized FT functional of the optimizer trajectory  in terms of the transition kernel of the optimizer. Similar results are obtained for continuous-time Markov models as well. Our result illustrates effective dimension reduction through properties (e.g., variance) of the kernel. 
\end{enumerate}
Finally, assuming that the behavior of the optimizer in a neighborhood of a local minimum can be well-approximated by a random walk, 
we bound the expected normalized FT functional in terms of the \emph{lower tail exponent} of the transition kernel, thereby obtaining our main result. Smaller values of the lower tail exponent correspond to tighter ``clustering" behavior in iterates of the optimizer (Figure~\ref{fig:BrownLevy}). Intuitively, this translates to a form of compression in the otherwise possibly high dimensional search space, leading to better generalization performance. 
While our analysis relies on the randomness in stochastic optimizers, for sufficiently large step sizes, deterministic optimizers (i.e., full-batch gradient descent) are known to also exhibit stochastic behaviors \cite{kong2020stochasticity}, and so our results extend to these cases. Our contributions here are predominantly theoretical, motivated by recent empirical work on state-of-the-art neural network models~\cite{martin2018implicit_JRNL,martin2020predicting_JRNL,martin2021simpsons}.
However, due to the relative tractability of estimating the local tail exponent, in Section \ref{sec:Experiments}, we demonstrate how these quantities correlates with generalization  in~practice.

\subsection{Related Work}
\label{sec:relwork}
For motivation and comparison, we discuss some important previous efforts and adjacent concepts in the literature. 

\ifdefined\nips
\textbf{Generalization bounds: }
\else
\paragraph{Generalization bounds.}
\fi
Naturally, there is a substantial literature involved in the development of ``generalization bounds,'' which we can only briefly summarize here. For more details, see \citet{jiang2019fantastic} and references therein. Almost universally, these bounds consider a ``single-point generalization gap,'' that is $|\mathcal{R}_n(w)-\mathcal{R}(w)|$, for fixed weights $w$. 
Earlier bounds were typically dependent only on properties of the \emph{model}, including Vapnik--Chervonenkis theory \cite{vapnik2015uniform}, and other norm-based bounds \cite{bartlett2017spectrally}. Most of these can be derived or sharpened through \emph{generic chaining} \cite{audibert2007combining}, which we shall also reconsider, albeit for a different problem. 
Such bounds are well-known to be vacuous \cite{jiang2019fantastic,bartlett2020failures,martin2021simpsons}. 
Non-vacuous bounds typically require some degree of \emph{data-dependence}, with the most effective of these bounds involving measures of \emph{sharpness} \cite{neyshabur2017exploring,jiang2019fantastic,martin2021simpsons}. 
Other bounds also have some degree of \emph{algorithm-dependence}, such as margin-based bounds \cite{antos2002data,sokolic2017generalization}, or bounds centered around stochastic gradient Langevin dynamics \cite{mou2018generalization,haghifam2020sharpened}. Specific to explaining the generalization in neural networks, which can even fit arbitrarily labeled data points~\cite{zhang2021understanding}, additional simplifying assumptions such as an infinite width~\cite{arora2018stronger}, a kernelized approximation~\cite{cao2020overparametrized}, or a simpler 2-layer setup~\cite{farnia2018a} are usually~made.

\ifdefined\nips
\textbf{Mutual information:} 
\else
\paragraph{Mutual information.}
\fi
One particular class of generalization bounds involves the \emph{mutual information} between the data and the stochastic optimizer, quantifying the one-point generalization gap by tying it to the learning ability of the algorithm itself \cite{russo2019much,xu2017information}. Intuitively, the mutual information balances the tradeoff between training loss and poor generalization due to overfitting. Such approaches are both data- and algorithm-dependent; and they can be made not only non-vacuous, but surprisingly tight \cite{asadi2018chaining}. %
Our Theorem \ref{thm:NormFL} will also involve mutual information, extending \cite{xu2017information} to bound the error (\ref{eq:AccError}). 
At present, applications of mutual information have tied variances in the optimizer to generalization \cite{pensia2018generalization,li2019generalization,negrea2019information,bu2020tightening,haghifam2020sharpened,neu21a}. Unfortunately, in larger models, the variance can anti-correlate with generalization \cite{jastrzkebski2017three,martin2020predicting_JRNL,martin2021simpsons}, which, to our knowledge, these analyses are unable to predict.

\ifdefined\icml
\else
\ifdefined\nips
\textbf{Two phases of learning:} 
\else
\paragraph{Two phases of learning.}
\fi
It is clear that stochastic optimizers typically undergo at least two distinct phases of learning \cite{li2019towards}: (1) an exploration-like ``catapult'' phase \cite{lewkowycz2020large}, when the step size is large, the optimizer moves rapidly between regions of the loss landscape, and more general patterns are fitted; and (2) an exploitation-like ``lazy training'' phase \cite{chizat2018lazy,fort2020deep}, when the step size is small, dynamics behave similarly to convex optimization around a central basin, and more precise patterns are fitted. Performance in this second phase is easily tied to mixing rates using arguments from convex optimization theory. However, very little is known surrounding the first phase, despite its significant apparent influence on generalization performance. Here, we provide a theoretical framework enabling investigation into \emph{both} of these phases.
\fi

\ifdefined\nips
\textbf{Optimization-based generalization:} 
\else
\paragraph{Optimization-based generalization.}
\fi
A body of work leans on implicit regularization effects of optimization algorithms to explain generalization, e.g., see~\citet{Arora2019implicit,chizat2020implicit} and references within for certain simplified problem settings. Another line of work focuses on stability of the optimization process~\cite{hardt16sgd} to bound the generalization gap. These are also one-point generalization bounds, and they do not take into account the trajectory of optimization.
In lieu of the inability of convex optimization theory to explain the behavior of SGD in non-convex settings, it is common to consider the behavior of Markov process models for stochastic optimizers \cite{mandt2016variational}.
These models are often continuous for ease of analysis \cite{orvieto2018continuous,simsekli2019tail}, although discrete-time treatments have become increasingly popular \cite{dieuleveut2017bridging,hodgkinson2020multiplicative,camuto2021fractal}. Such continuous-time models are formulated as stochastic differential equations $\dd W_t = \mu(W_t) \dd t + \sigma(W_t) \dd X_t$, where $X_t$ is typically Brownian motion, or some other L\'{e}vy process, and derived through the (generalized) central limit theorem and taking learning rates to zero \cite{fontaine2020continuous}. 

\ifdefined\nips
\textbf{Heavy-tailed universality:} 
\else
\paragraph{Heavy-tailed universality.}
\fi
Recent investigations have identified the presence of heavy tails in the dynamics of stochastic optimizers \cite{simsekli2019tail,simsekli2019heavy, panigrahi2019non}. Subsequent theoretical analyses trace the origins of these fluctuations to the presence of multiplicative noise \cite{hodgkinson2020multiplicative,gurbuzbalaban2020heavy}. Establishing theory connecting generalization performance to the presence of power laws has become a prominent open problem in light of the empirical and theoretical studies of  \citet{martin2017rethinking,martin2018implicit_JRNL,martin2019traditional,martin2020heavy,martin2020predicting_JRNL}; these studies have explicitly tied performance to the presence of heavier tails in the spectral distributions of weights, which was then linked to generalization through compressibility, under a statistical independence assumption \cite{barsbey2021heavy}. 
Of particular note is the previous work of \citet{simsekli2020hausdorff}, correlating generalization performance with heavier-tailed dynamics; this work considered (\ref{eq:AccError}) in the case $(t_1,t_2) = (0,1)$ and in the context of continuous-time stochastic optimizer models. However, their approach is bound to a continuous-time Feller process model for the optimizer; and a significant objective of this work is to extend these ideas into the natural discrete-time~setting.

\section{Preliminaries}
\label{sec:FT}
\subsection{Background}
Let $\ell:\mathbb{R}^D \times \mathbb{R}^p \to \mathbb{R}_+$ be a non-negative loss function assessing accuracy for a model with parameters $w \in \mathbb{R}^D$ to fixed data $X \in \mathbb{R}^p$. Total model accuracy is determined by the population risk function $\mathcal{R}(w) \coloneqq \mathbb{E}_{X \sim \mathcal{D}} [\ell(w,X)]$, where $\mathcal{D}$ denotes the distribution of possible data. Therefore, the optimal choice of parameters is determined to be those solving the true risk minimization problem $\min_w \mathcal{R}(w)$. As this problem is intractable, model training is typically achieved by solving the empirical risk minimization problem: for a collection of data $X_1,\dots,X_n \iidsim \mathcal{D}$, solve
\[
\min\nolimits_w \mathcal{R}_n(w),\quad \mbox{where}\quad \mathcal{R}_n(w) \coloneqq (1/n) \sum\nolimits_{i=1}^n \ell(w,X_i)
\]
is the empirical risk function. Provided this problem can be solved to near-zero empirical risk, model accuracy depends on the \emph{generalization gap} $\mathcal{E}_n(w) = n^{-1} \sum_{i=1}^n R_i(w)$, where $R_i(w) = \ell(w,X_i) - \mathbb{E}_{X \sim \mathcal{D}} \ell(w,X)$.  

Given a set of parameters $W \subset \mathbb{R}^D$, our objective is to bound the worst-case generalization gap over all $w \in W$, i.e., $\sup_{w\in W} |\mathcal{E}_n(w)|$. The set $W$ is kept arbitrary for now, yet, we will be mainly interested in the case where $W$ is part of an optimizer trajectory, e.g., $W = \{W_t\}_{t=t_1}^{t_2}$ (cf.\ \eqref{eq:AccError}).
\ifdefined\icml 
\else
Assuming that $\ell$ is bounded in magnitude by $B > 0$, from McDiarmid's inequality, with probability at least $1 - \delta$, it holds that
$$
\max_{w \in W}|\mathcal{E}_n(w)| \leq \mathbb{E}\max_{w \in W}|\mathcal{E}_n(w)| + B\sqrt{\frac{\log(1/\delta)}{n}}.
$$
In the case where $W$ is finite, a naive union bound implies 
\begin{equation}
\label{eq:NaiveBound}
\mathbb{E}\max_{w \in W}|\mathcal{E}_n(w)| \leq B \sqrt{\frac{\log(2|W|)}{2 n}},
\end{equation}
which holds regardless of the properties or geometry of $W$. Up to constants, provided $\log (2|W|) = o(n)$, this is the best that one can achieve by assuming \emph{only} that the loss is bounded (see \citet[Theorem 5.3.3]{talagrand2014upper}, for~example).
\fi

It is common to assume that the loss functions are Lipschitz in $w$, that is,  %
for some metric $(x,y)\mapsto d(x,y)$ on $\mathbb{R}^D$, %
$|\ell(w,X)-\ell(w',X)| \leq Ld(w,w')$ for $w,w'\in W$. Hoeffding's inequality \citep[Theorem 2.8]{boucheron2013concentration} then implies that
the difference $R_i(w)-R_i(w')$ between two points $w,w' \in W$ is sub-Gaussian with variance parameter $(Ld(w,w'))^2$. %
Dudley's classical method of chaining \cite{dudley1967sizes} asserts that one can take advantage of the 
triangle inequality to ``chain'' these bounds together over particular choices of $w,w'$ and to bound the maximal error over the set $W$. \citet{talagrand1996majorizing} later improved on this approach and developed \emph{generic chaining}. This approach is heavily inspired by the following functional originally introduced by \citet{fernique1971regularite,fernique1975regularite}: 
\begin{equation}
\label{eq:FerniqueTalagrand}
\gamma_2(W,d) = \inf_{\mu} \sup_{w \in W} \int_0^{\diam(W)} \sqrt{\log \frac1{\mu(B_r^d(w))}} \dd r,
\end{equation}
where $B_r^d(w) = \{w': d(w,w') \leq r\}$ is the ball of radius $r$ under $d$ around $w$, $\diam(W) = \sup_{w,w'\in W} d(w,w')$, and the infimum is taken over all probability measures $\mu$ supported on $W$. We refer to the functional (\ref{eq:FerniqueTalagrand}) as the \emph{Fernique--Talagrand (FT) functional}. Generic chaining would later focus on an (equivalent, up to constants) discrete variant of this functional \cite{talagrand2001majorizing,audibert2007combining,talagrand2014upper}, but for our purposes, we will find it more convenient to work with the original formulation (\ref{eq:FerniqueTalagrand}).

It is also known that the approach is essentially \emph{optimal} in the following sense: if the error $w \mapsto \mathcal{R}_n(w) - \mathcal{R}(w)$ is a Gaussian process with covariance $(w,w')\mapsto (Ld(w,w'))^2$, then the FT functional both upper and lower bounds the maximal expected error in the empirical risk $\mathbb{E}\sup_{w \in W}|\mathcal{R}_n(w)-\mathcal{R}(w)|$ up to constants \citep[Theorem 5.1]{talagrand1996majorizing}. Therefore, in the absence of additional information on the distribution of the empirical risk, the FT functional provides the \emph{sharpest possible generalization bound} up to constant~factors.

Broadly speaking, the FT functional simultaneously measures \emph{variance} and \emph{clustering}. Clustering occurs in the absence of spatial homogeneity, and it can be measured in a number of ways. For more discussion, we refer to Appendix \ref{sec:VarCluster}. %
Later, to draw connections to heavy-tailed theory in machine learning, e.g. \cite{simsekli2020hausdorff}, the degree of clustering will be represented using lower tail exponents in the transition kernel. %

\subsection{Data-dependence with mutual information}
In the classical setting, the set $W$ is fixed and deterministic. 
However, in our setting of (\ref{eq:AccError}), $W$ is both random and data-dependent\footnote{We refer the reader to \cite{molchanov2005theory} for the definition and details of a random set.}. To extend the theory to allow for data-dependent $W$, we shall invoke some ideas from information theory. Recall that the $\alpha$-Renyi divergence is defined by 
\[
D_{\alpha}(\mu,\nu)=\frac{1}{\alpha-1}\log\mathbb{E}\bigg[\frac{\dd \mu}{\dd \lambda}(Z)^\alpha \frac{\dd \nu}{\dd \lambda}(Z)^{1-\alpha}\bigg],
\]
where $\alpha>1$ and $Z$ is distributed according to some (arbitrary) probability measure $\lambda$ where $\mu$ and $\nu$ are absolutely continuous with respect to $\lambda$ (for example, $\lambda=\frac{1}{2}(\mu+\nu)$).
The $\alpha$-mutual information between two random elements $X,Y$ is defined as the $\alpha$-Renyi divergence between the joint probability measure $\mathbb{P}_{X,Y}$ and the product measure $\mathbb{P}_{X}\otimes\mathbb{P}_{Y}$: $I_{\alpha}(X,Y)=D_{\alpha}(\mathbb{P}_{X,Y}\Vert\mathbb{P}_{X}\otimes\mathbb{P}_{Y})$, measuring the extent of the dependence between $X$ and $Y$. 
The standard Kullback-Leibler mutual information is obtained by taking $\alpha\to1^{+}$. The $\alpha$-mutual information is non-decreasing in $\alpha$, that is, $I_{\alpha}(X,Y)\leq I_{\beta}(X,Y)$ for $\alpha\leq\beta$. Therefore, we may define the \emph{total mutual information} as $I_{\infty}(X,Y)=\lim_{\alpha\to\infty}I_{\alpha}(X,Y)=\sup_{\alpha}I_{\alpha}(X,Y)$. 

\subsection{Normalized Fernique--Talagrand functional}
\label{sec:ftnorm}

It is often the case with generalization bounds that $\ell$ is assumed to be Lipschitz-continuous with respect to the Euclidean metric \citep{neyshabur2017exploring,mou2018generalization}. To also ensure subgaussianity of $R_i(w)$ itself, boundedness of $\ell$ is often assumed \citep{negrea2019information,simsekli2020hausdorff}. Together, these two assumptions are equivalent to assuming Lipschitz continuity under the truncated metric $d_\rho(x,y) = \min\{\rho,\|x-y\|\}$, where $\rho > 0$. We refer to the corresponding FT functional $\gamma_2^\rho(W) \coloneqq \gamma_2(W,d_\rho)$ as the \emph{normalized} Fernique--Talagrand functional: for $\rho > 0$, from (\ref{eq:FerniqueTalagrand}) and scaling by $1/\rho$,
\begin{equation}
\label{eq:NormFT}
\gamma_2^\rho(W) = \inf_{\mu}\sup_{w\in W}\frac{1}{\rho}\int_{0}^\rho\sqrt{\log\frac{1}{\mu(B_{r}(w))}}\dd r,
\end{equation}
where $B_r(w)$ denotes the Euclidean ball of radius $r > 0$ about $w$, and the infimum is once again over all probability measures $\mu$ on $W$. %
The functional is considered normalized as it does not grow as $\diam(W) \to \infty$. Here, $\rho$ is arbitrary, although we will find the tightest bounds to occur when $\rho = B / L$, where $B$ and $L$ is the upper bound, and Lipschitz constant of $\ell$, respectively. It is worth noting that in the analysis to follow, the bounded+Lipschitz assumption could be relaxed (e.g. to H\"{o}lder continuity) by choosing a different metric, thus considering a different FT functional. However, we have found this choice of assumption and its corresponding FT functional (\ref{eq:NormFT}) to yield the best compliance with our experiments. 

\section{Main Results}
\label{sec:main}

Our first main result is presented in Theorem \ref{thm:NormFL} below. If $W$ is uncountable, we interpret probabilities and expectations of suprema over uncountable sets as the corresponding supremum over all possible countable subsets. %
\begin{theorem}
\label{thm:NormFL}
Assume that $\ell$ is bounded by $B > 0$ and $L$-Lipschitz continuous (with respect to the Euclidean metric). There exists a universal constant $K_1 > 0$ such that for any (random) closed set $W \subset \mathbb{R}^D$, with probability at least $1 - \delta$, for any $\rho > 0$, letting $L_\rho := \max\{B,L\rho\}$,
\begin{multline}
\label{eq:GenBoundProb}
\sup_{w \in W} |\mathcal{E}_n(w)| \\ \leq K_1 L_\rho \left(\frac{\gamma_2^\rho(W)}{\sqrt{n}} + \sqrt{\frac{\log(1/\delta) + I_\infty(X, W)}{n}}\right).
\end{multline}
Furthermore, there exists %
$K_2 > 0$ such that
\begin{equation}
\label{eq:GenBoundExp}
\mathbb{E}\sup_{w \in W} |\mathcal{E}_n(w)| \leq K_2 L_\rho \left(\frac{\mathbb{E}\gamma_2^\rho(W) + \sqrt{I_1(X,W)}}{\sqrt{n}}\right).
\end{equation}
If $\ell$ is unbounded, then (\ref{eq:GenBoundProb}) holds with probability at least $1 - \delta - \mathbb{P}(\sup_{w \in W} \mathcal{R}_n(w) > B)$. 
\end{theorem}

Note also that when $W = \{w^\ast\}$, where $w^\ast$ is the location of the optimizer at some \emph{deterministic} stopping time, we recover (up to constants) the information-theoretic bound of \citet{xu2017information}. Therefore, we inherit the interpretation of mutual information as a measurement of ``overfitting,'' together with its follow-up developments \cite{asadi2018chaining, haghifam2020sharpened}. The advantage of our bound is that we can better investigate the effect of dynamics on generalization through the supremum over the trajectory. We also inherit the sharpness of Theorem \ref{thm:NormFL} from that of the FT functional in the event that $X$ and $W$ are independent.

\label{sec:Bounds}

\subsection{Markov stochastic optimizers}

Our remaining theoretical contributions are concerned with bounding and estimating $\gamma_2^\rho$ when applied to the trajectory of a stochastic optimizer. Such bounds on $\gamma_2^\rho$ directly imply generalization bounds by applying Theorem \ref{thm:NormFL}. 

We adopt the Markov formulation of stochastic optimizers, seen in \citet{hodgkinson2020multiplicative}. This formulation incorporates SGD, momentum, Adam, and stochastic Newton, among others. To summarize, approximate solutions to problems of the form $\argmin_w \mathbb{E}\ell(w,X)$ are typically obtained by fixed point iteration: for some continuous map $\Psi$ such that any fixed point of $\mathbb{E}\Psi(\cdot,X)$ is a minimizer of $\ell$, a stochastic optimizer is constructed from the sequence of iterations $W_{k+1} = \Psi(W_k, X_{k+1})$, where $X_k$ are independent copies of $X$. For example, (online) stochastic gradient descent corresponds to the choice $\Psi(w,x) = w - \gamma b^{-1} \sum_{i=1}^b\nabla \ell(w, x_i)$, where $\gamma$ is a chosen learning rate and $b$ denotes the batch size. These iterations induce a discrete-time Markov chain in $W_k$ with transition kernel $P(w,E) = \mathbb{P}(W_{k+1} \in E\,\vert \,W_k = w)$ for a measurable set $E$. Under certain regimes (for example, small learning rate and large batch size), this chain is well-approximated by a continuous-time Markov process $\{W_t\}_{t \geq 0}$ with transition kernel $P_t(w, E) = \mathbb{P}(W_t \in E \,\vert\, W_0 = w)$ --- see \citet{fontaine2020continuous}, for example. We note that imposing a Markov assumption on the optimizer is not restrictive, as any recursive method is Markov under suitable state augmentation. Furthermore, the Markov assumption is local, so $W_0$ may be taken to be any point in the optimization.

Fortunately, the FT functional is sufficiently versatile that we can provide a direct generalization bound in terms of transition kernels. %
This is accomplished using covering arguments and the classical Dudley entropy bound \citep[Corollary 13.2]{boucheron2013concentration}. 
Here is our main result for this.

\begin{theorem}
\label{thm:bound_transition_kernel}
For a Markov transition kernel $P(x,E)$ and any $\rho > 0$, let
\[
\mathcal{I}_\rho[P] := \frac{1}{\rho} \int_0^\rho \sup_{x \in \mathbb{R}^D} \sqrt{\log \left(\frac{3^{D+2}}{P(x,B_r(x))} \right) } \dd r.
\]
There exists a universal constant $K > 0$ such that the following bounds on the FT functional hold:
\begin{enumerate}[leftmargin=*]
\item Let $W_k$, $k \geq 0$ be a discrete-time homogeneous Markov chain on $\mathbb{R}^D$ with transition kernel $P^k(w,E)$. Then for the average kernel $\bar{P}_m(w,E) := m^{-1} \sum_{k=1}^m P^k(w,E)$, 
we have $\mathbb{E}\gamma^\rho_2(\{W_k\}_{k=0}^m) \leq K \mathcal{I}_{\rho}[\bar{P}_m]$.
\item Let $\{W_t\}_{t \in [0,T]}$ be a continuous-time homogeneous Markov process on $\mathbb{R}^D$ with kernel $P_t(w,E)$. Then for the average kernel $\bar{P}_T(w,E) = \frac1T \int_0^T P_t(w,E) \dd t$, we have 
$\mathbb{E}\gamma^\rho_2(\{W_t\}_{t\in[0,T]}) \leq K \mathcal{I}_{\rho}[\bar{P}_T]$. 
\end{enumerate}
\end{theorem}
We can use Theorems~\ref{thm:NormFL} and \ref{thm:bound_transition_kernel} to bound the generalization gap for a stochastic process as long as we can characterize the transition probability kernel of the process. %
Due to the dependence on the ambient dimension, we do not expect the bounds in Theorem \ref{thm:bound_transition_kernel} to be sharp\footnote{We suspect this dependence can be removed using a more efficient embedding, but we leave this as an open problem. For further details, refer to the comments in the proof.}. Fortunately, this is no concern for our purposes, as we will use Theorem \ref{thm:bound_transition_kernel} to imply correlations between tail properties of the transition kernel and the generalization gap. 

\subsection{Fractal dimensions}

It has been observed by \citet{simsekli2020hausdorff,birdal2021intrinsic} that the \emph{fractal dimension} of the set $W$ is often a good indicator of generalization performance. 
While this previous approach focused on precise covering arguments, here we show that a similar bound to that of \citet[Theorem 2]{simsekli2020hausdorff} can be readily attained using Theorem \ref{thm:NormFL} under weaker assumptions. Indeed, Theorem \ref{thm:NormFL} provides a remarkably straightforward illustration of the relationship between generalization performance and fractal dimension: Assuming there exists a measure $\mu$ on $W$ such that $\mu(B_r(w)) \geq (c r)^{\alpha}$ for any $w \in W$ and $0 < r < \rho$, for some $c > 0$, $\alpha > 0$, then inserting this measure into the definition (\ref{eq:NormFT}) reveals $\gamma_2^\rho(W) = \mathcal{O}(\sqrt{\alpha})$. If a similar upper bound also holds for $\mu$, then most notions of fractal dimension of the set $W$ coincide, and are \emph{precisely equal to} $\alpha$ \citep[Chapter 5]{mattila1999geometry}. In this way, $\alpha$ becomes an \emph{effective dimension}, or \emph{intrinsic dimension}, of $W$ (in particular, if $W \subset \mathbb{R}^D$, then $\alpha \leq D$, the ambient dimension). This idea is formalized in Corollary \ref{cor:Hausdorff}. The result involves the \emph{Hausdorff dimension} of $W$, i.e., $\dimh W \in [0,D]$, which is a generalization of the usual notion of dimension to fractional orders (e.g., $\dimh \mathbb{R}^D = D$), and the \emph{Hausdorff measure} $\mathscr{H}^\alpha$, which is a generalization of the Lebesgue measure. We provide their exact definitions in Appendix \ref{sec:HausdorffApp}.

\begin{corollary}
\label{cor:Hausdorff}
Suppose that $\dimh W = \alpha$ and is $\alpha$-Ahlfors lower regular almost surely, 
that~is, 
\[
C_\rho \coloneqq \inf_{0 < r < \rho, w \in W} \frac{\mathscr{H}^\alpha(W \cap B_r(w))}{r^\alpha \mathscr{H}^\alpha(W)} > 0.
\]
Then the normalized Fernique--Talagrand functional satisfies $\gamma_2^\rho(W) \leq (2 \rho C_\rho)^{-1} \sqrt{\pi \alpha}$ almost surely.
\end{corollary}
Note that the Ahlfors regularity assumption (i.e., $C_\rho > 0$) is contained in \citet[Assumption H4]{simsekli2020hausdorff}, but is a far simpler condition on its own, and does not require a Feller process construction. Furthermore, we improve on \citet[Theorem 2]{simsekli2020hausdorff}, which suggests a rate $\mathcal{O}(\log n /\sqrt{n})$, whereas we obtain a rate of $\mathcal{O}(1/\sqrt{n})$. 

Using a continuous-time Markov model, a precise link between Corollary \ref{cor:Hausdorff} and stochastic optimization can be made: suppose that $W_t$, $t \in [0,1]$, is a continuous-time Markov process with transition kernel $P_t(x,E)$, e.g., a continuous-time model of a stochastic optimizer \cite{orvieto2018continuous}.  %
If $W_t$ is spatially homogeneous\footnote{There exists $K$, $C$, $r_0 >0$ such that $P_t(x, B_{r_{0}}\left(x)\right) \geq K$ for all $t$, $x$, and
$\frac1{C}P_t(0, B_r(0)) \leq P_t(x, B_r(x)) \leq C P_t(0, B_r(0))$ for all $t$, $x$, $r \leq r_{0} $.}, 
and $\{W_t\}_{t\in[0,1]}$ is $\alpha$-Ahlfors lower-regular almost surely, then by \citet[Theorem 4.2]{xiao2003random}, Corollary \ref{cor:Hausdorff} applies with $\dimh \{W_t\}_{t\in [0,1]} = \alpha$ almost surely, where
\begin{equation}
\label{eq:XiaoExponent}
\small\alpha = \sup\left\{\gamma \geq 0 : \lim_{r\to 0^+} \hspace{-2pt} \frac1{r^{\gamma}} \int_0^1 \hspace{-3pt} P_t(0, B_r(0)) \dd t < \infty\right\}.
\end{equation}
Similar results also apply for spatially inhomogeneous Feller processes \citep{schilling1998feller}. %
Therefore, Corollary \ref{cor:Hausdorff} supports the claim that fractal dimensions (of the trajectory of a continuous Markov model of a stochastic optimizer) can be an effective measure of generalization performance. This is the strategy proposed by \citet{simsekli2020hausdorff}. However, in reality, optimization procedures are not continuous-time and optimizer trajectories are finite sets. Since all fractal dimensions are identically zero on finite sets, an alternative approach is required.%

\subsection{Lower tail behavior around local minima}
\label{sec:LowerTail}

Fortunately, we can obtain a discrete-time analogue of Corollary \ref{cor:Hausdorff} using Theorem \ref{thm:bound_transition_kernel}. 
There is one significant caveat however: while Theorem \ref{thm:bound_transition_kernel} successfully relates the dynamics of a Markov stochastic optimizer to generalization performance, the bound is ineffective when the increments $W_{k+1}-W_k$ are not uniformly stochastically bounded. %
Indeed, the bound in Theorem \ref{thm:bound_transition_kernel} is most tight when the optimizer exhibits random walk behavior without drift. This behavior is unlikely to occur at the global scale. However, \emph{locally}, in the neighborhood of a local minimum, %
a stochastic optimizer should exhibit minimal drift. Furthermore, in practice, the behavior around a minimum is typically of greatest interest. Recall that the first objective of a stochastic optimizer is to reach and then occupy some central region $\Omega$ around a local minimum $x^\ast \in \Omega$ with high probability. We let $\zeta_m$ denote the probability of remaining in this region after $m$ steps. Within this region, we assume that the optimizer behaves like a random walk $\bar{W}_{k+1} = \bar{W}_k + Z_k$, where each $Z_1,\dots,Z_k \iidsim \mu$ is independent and identically distributed. To incorporate these observations into a bound, we can appeal to approximation in total variation $d_{\text{TV}}$. Under these conditions, the assumption of bounded $\ell$ is no longer restrictive, as $\ell$ can be assumed to be only \emph{locally} bounded. 

With this in mind, we develop a discrete time analogue of Corollary \ref{cor:Hausdorff}, closing an open problem connecting tail exponents to generalization. %
Drawing inspiration from (\ref{eq:XiaoExponent}), we define a new $\alpha$ such that the kernel $P(x, B_r(x)) \approx c r^{\alpha}$ as $r \to 0^+$. In particular, we can let
\[
\alpha = \lim_{r \to 0^+} \frac{\log P(x,B_r(x))}{\log r}.
\]
This way, $\alpha$ becomes an exponent on the \emph{lower tail} of the transition kernel. Note that, \emph{a priori}, this is \emph{distinct from the (upper) tail exponent} of the distribution considered in \citet{hodgkinson2020multiplicative,simsekli2019tail,simsekli2020hausdorff}, although the two appear to correlate in practice 
(see Section \ref{sec:StableExp}).
\begin{corollary}
\label{cor:BulkExponent}
Let $\Omega \subseteq \mathbb{R}^D$ be a closed set such that $\mathbb{P}(W_k \notin \Omega \text{ for some } k = 0,\dots,m) \leq \zeta_m$ and let $\mu$ be a probability measure on $\mathbb{R}^D$. Suppose that $Z_1,\dots,Z_m \overset{\text{iid}}{\sim} \mu$ and there exists $r_0,\alpha > 0$ such that $\mathbb{P}(\|Z_1+\cdots+Z_k\|\leq r) \geq c_k r^\alpha$ for all $0 < r < r_0$, where $c_k > 0$ are constants for each $k=1,2,\dots$. Letting $P_\Omega$ denote the transition kernel of $W_k$ conditioned on $W_k \in \Omega$ for $k=1,2,\dots,m$,
there is a constant $K > 0$ such that for any $\epsilon > 0$, there exists $\rho_\epsilon > 0$ independent of $\alpha$ where
\begin{multline}
\label{eqn:cor2}
\mathbb{E}\gamma_{2}^{\rho_\epsilon}(\{W_{k}\}_{k=0}^{m})\leq \frac{K}{\rho_\epsilon}\sqrt{\alpha + \epsilon} +\\ \sqrt{\log(m+1)}\left(\zeta_{m}+m\sup_{x\in\Omega}d_{\textnormal{TV}}(P_{\Omega}(x,x+\cdot),\mu)\right).    
\end{multline}
\end{corollary}

The second line of (\ref{eqn:cor2}) concerns only the quality of the random walk approximation, and can mostly be ignored if we expect such an approximation to be accurate (e.g., the example below). What remains is an implied correlation between the expected normalized FT functional (itself linked to generalization through Theorem \ref{thm:NormFL}) and the lower tail exponent $\alpha$. This can be seen in Figure \ref{fig:BrownLevy}, where a L\'{e}vy process with lower tail exponent $\alpha = 1.5$ is compared to Brownian motion with exponent $\alpha = 2$. The reduced tail exponent coincides with a corresponding reduction in $\gamma_2^1$.

\begin{example*}[Perturbed Gradient Descent]
Arguably the most common %
discrete-time model for a stochastic optimizer is the perturbed gradient descent (GD) model, which satisfies
$W_{k+1} = W_k - \gamma (\nabla f(W_k) + Z_k)$, where $Z_k$ is a Gaussian vector with zero mean and constant covariance matrix $\Sigma$ representing noise in the stochastic gradient. In the neighborhood of a local optimum $w^\ast$, $\nabla f(W_k) \approx 0$, and hence the perturbed GD model resembles a Gaussian random walk $W_{k+1} = W_k + \gamma Z_k$. Here, the exponent $\alpha$ is \emph{precisely} the ambient dimension, i.e., $\alpha = D$ 
(see Appendix \ref{sec:GrowthExp}). 
However, we will find this exponent to be much less than the ambient dimension for an actual optimization path, suggesting an interpretation of $\alpha$ as a measure of \emph{effective dimension} for the purposes of generalization.
\end{example*}

\section{Empirical results}
\label{sec:Experiments}

\ifdefined\nips
\vspace{-5pt}
\fi

\subsection{Lower tail exponents of the transition kernel}

For our experiments, we consider three architectures and two standard image classification datasets. 
In particular, we consider (i) a fully connected model with $5$ layers (FCN5), (ii) a fully connected model with $7$ layers (FCN7), and (iii) a convolutional model with $9$ layers (CNN9); and two datasets (i) MNIST and (ii) CIFAR10. All models use the ReLU activation function and all are trained with constant step-size SGD, without weight-decay or momentum. Our code is implemented in \texttt{PyTorch} and executed on $5$ GeForce GTX 1080 GPUs. For each architecture, we trained the networks with different step-sizes and batch-sizes, where we varied the step-size in the range $[0.002, 0.35]$ and the batch-size in the set $\{50,100\}$.
We trained all models until training accuracy reaches exactly $100\%$. 
For measuring training and test accuracies, we use standard training-test splits.

\ifdefined\nips
\begin{wrapfigure}{r}{0.675\textwidth}
\centering
\includegraphics[width=0.22\textwidth]{figures/5_00_fcn_mnist_pl.png}
\includegraphics[width=0.22\textwidth]{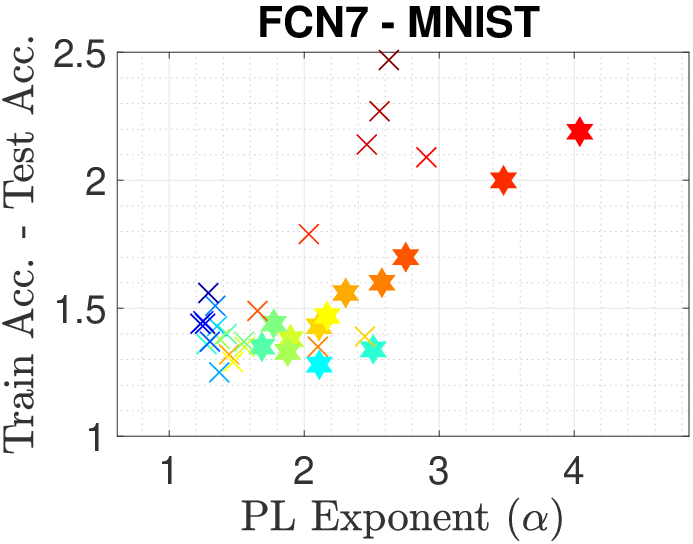}
\includegraphics[width=0.22\textwidth]{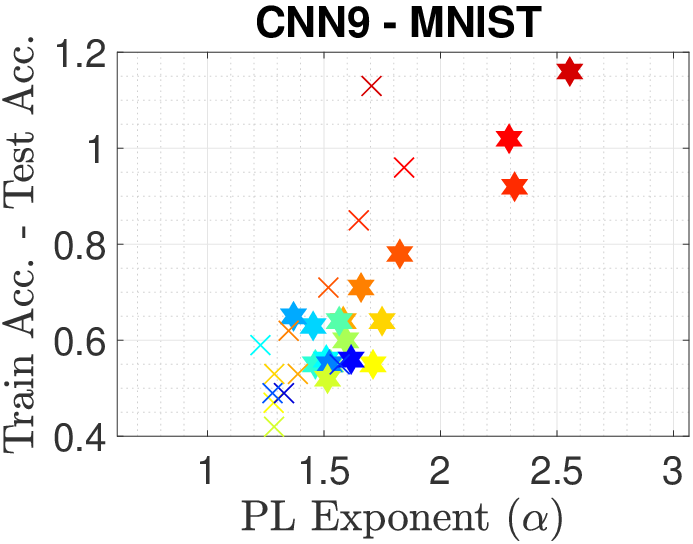}\\
\includegraphics[width=0.22\textwidth]{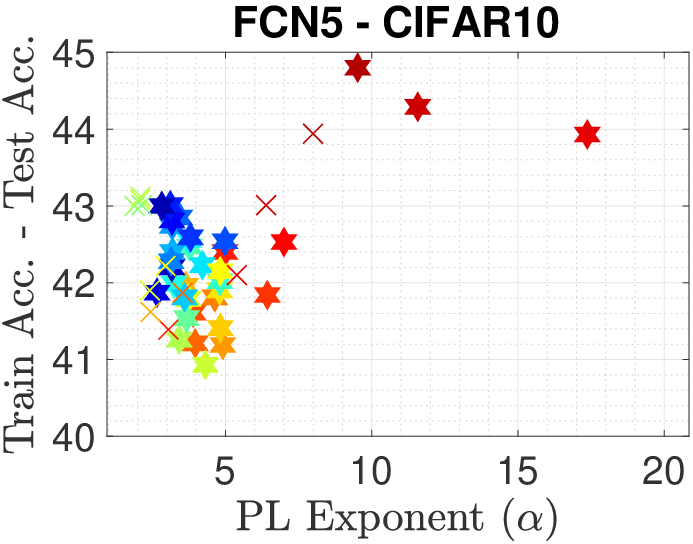}
\includegraphics[width=0.22\textwidth]{figures/7_00_fcn_cifar10_pl.png}
\includegraphics[width=0.22\textwidth]{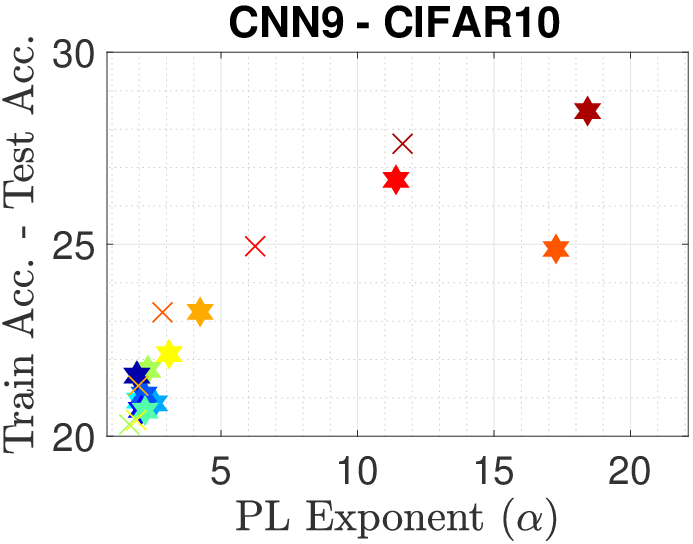}
\caption{Lower tail exponents versus generalization gap. Different colors represent different step-sizes and different markers represent different batch-sizes.
\label{fig:exps_pl}
}
\vspace{-10pt}
\end{wrapfigure}
\else
\begin{figure}
\centering
\includegraphics[width=0.23\textwidth]{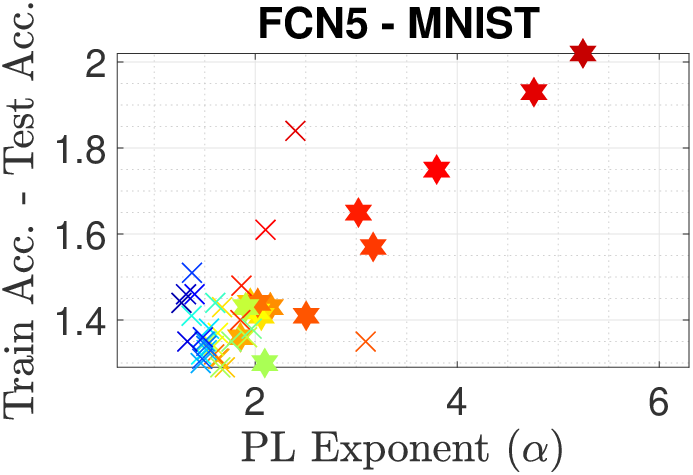}
\includegraphics[width=0.23\textwidth]{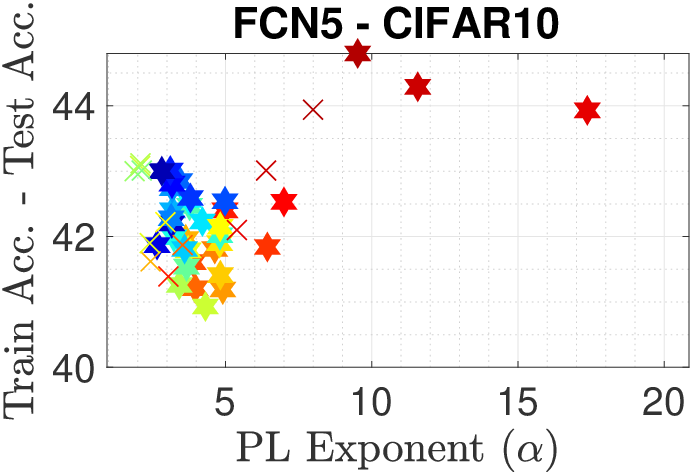}
\includegraphics[width=0.23\textwidth]{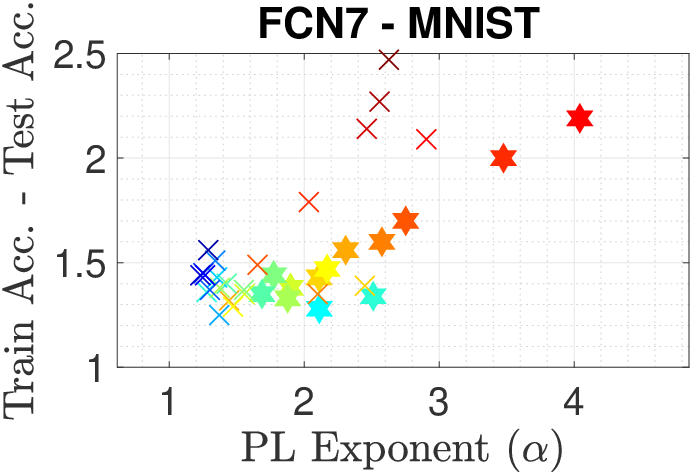} 
\includegraphics[width=0.23\textwidth]{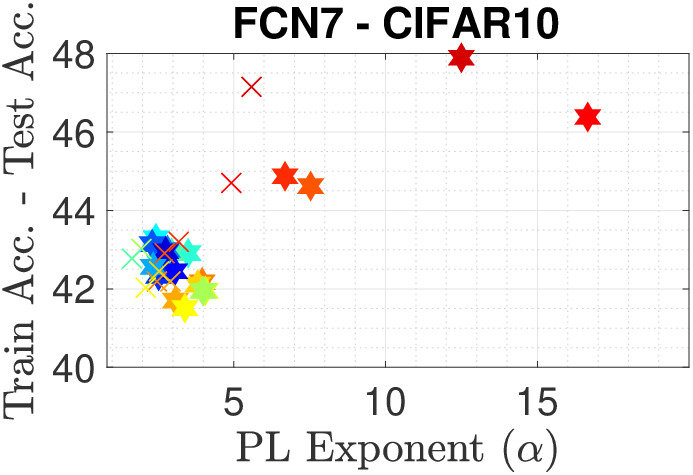}
\includegraphics[width=0.23\textwidth]{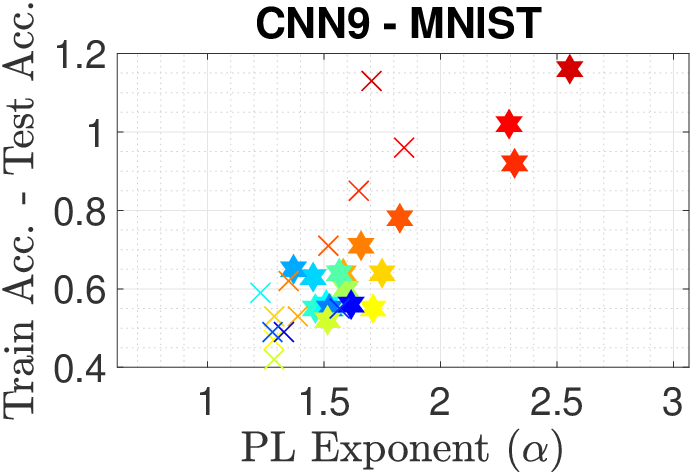}
\includegraphics[width=0.23\textwidth]{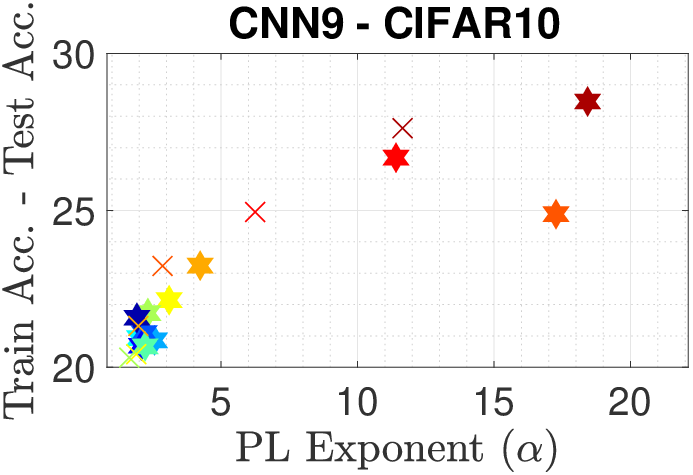}
\caption{Lower tail exponents versus generalization gap. Different colors represent different step-sizes and different markers represent different batch-sizes.
\label{fig:exps_pl}
}
\vspace{-.5cm}
\end{figure}
\fi

To estimate $\alpha$, once training accuracy reaches 100\%, we further run the algorithm for $m=200$ additional iterations to obtain a trajectory $\{W_k\}_{k=1}^{m}$. Following Corollary~\ref{cor:BulkExponent}, we assume \emph{local homogeneity}, that is,
the trajectory $\{W_k\}_{k=1}^{m}$ remains near the local minimum %
and each $W_{k+1}-W_k$ is approximately iid.
Under this assumption, the second term in \eqref{eqn:cor2} can be ignored, hence, we compute the sequence $1/\|W_{k+1}-W_k\|$ for $k=1,\dots,m-1$, and then fit a power-law (PL) distribution to this one dimensional set of observations, by using the $\mathtt{powerlaw}$ toolbox   \citep{clauset2009power}. Figure~\ref{fig:exps_pl} visualizes the results. In all configurations, we observe that the results are in  accordance with our theory: the estimated lower tail exponents $\alpha$ and the generalization gap exhibit significant correlations. 
This trend is even clearer for the CNN9 model, suggesting the geometry induced by the convolutional architecture results in a transition kernel for which the local homogeneity condition becomes more accurate, i.e., the TV term in Corollary~\ref{cor:BulkExponent} becomes smaller.
Furthermore, for the significance of correlations in Figure~\ref{fig:exps_pl}, we compute the $p$-values under a linear model. The results show that the PL exponent and generalization error are always positively correlated and this correlation is significant: the $p$-value ranges from $10^{-12}$ to $10^{-4}$.

\ifdefined\nips
\begin{figure}[t] %
\centering
\includegraphics[width=0.161\textwidth]{figures/5_00_fcn_mnist_pl_etab.png}
\includegraphics[width=0.161\textwidth]{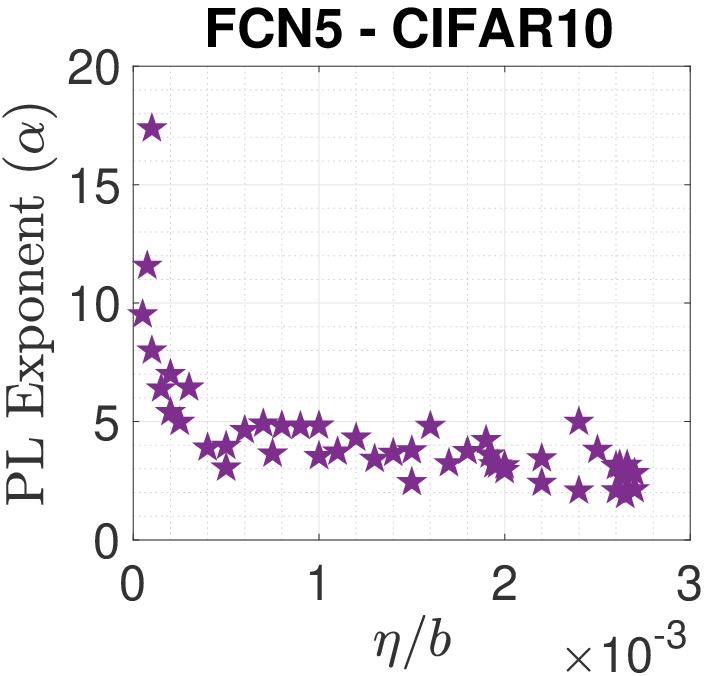}
\includegraphics[width=0.161\textwidth]{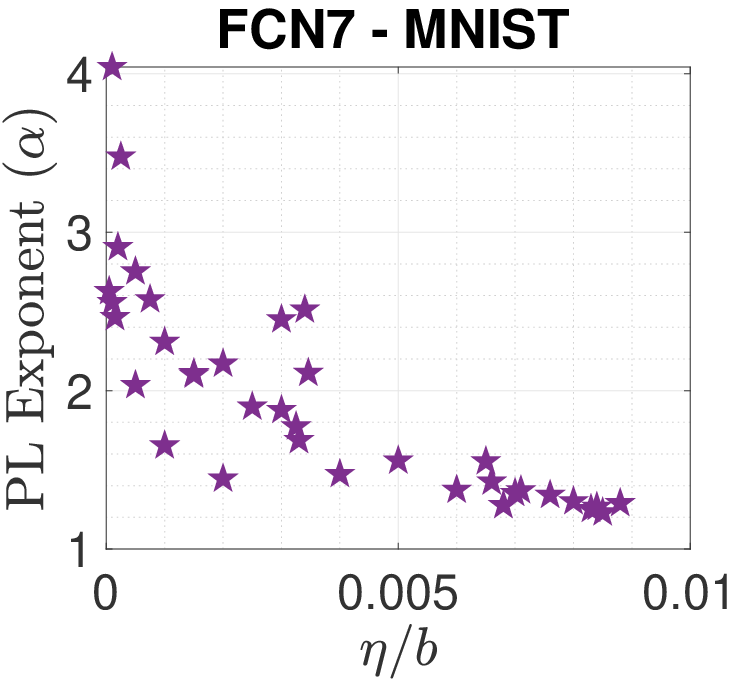}
\includegraphics[width=0.161\textwidth]{figures/7_00_fcn_cifar10_pl_etab.png}
\includegraphics[width=0.161\textwidth]{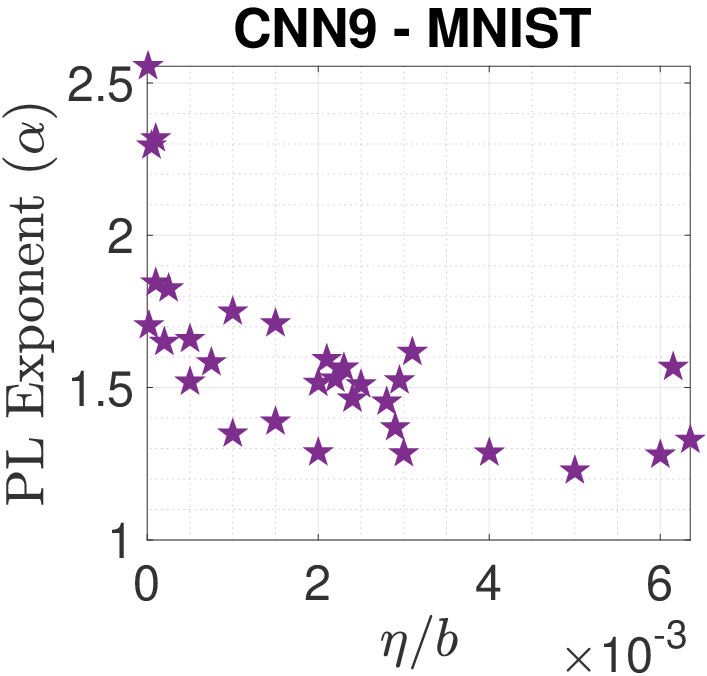}
\includegraphics[width=0.161\textwidth]{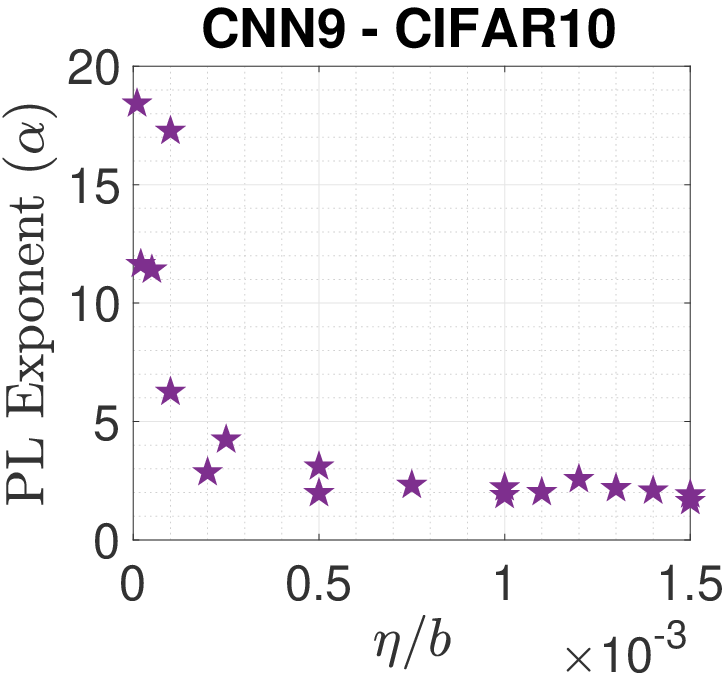}
\vspace{-10pt}
\caption{Lower tail exponents versus step-size/batch-size ratio ($\eta/b$).\label{fig:exps_pl_etab}}

\vspace{-10pt}
\end{figure}
\else
\begin{figure}[t] %
\centering
\includegraphics[width=0.23\textwidth]{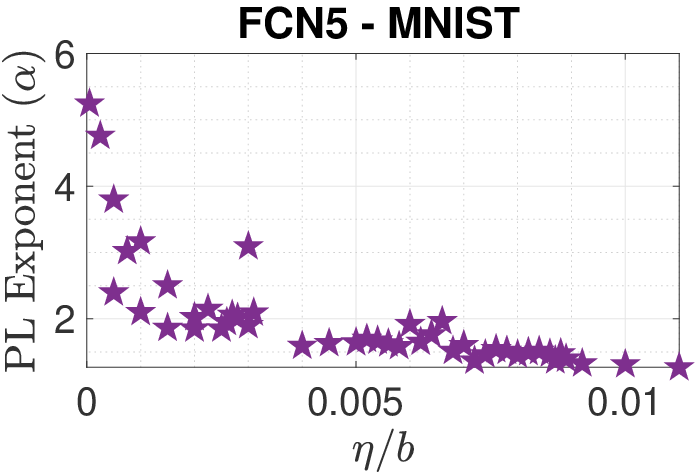}
\includegraphics[width=0.23\textwidth]{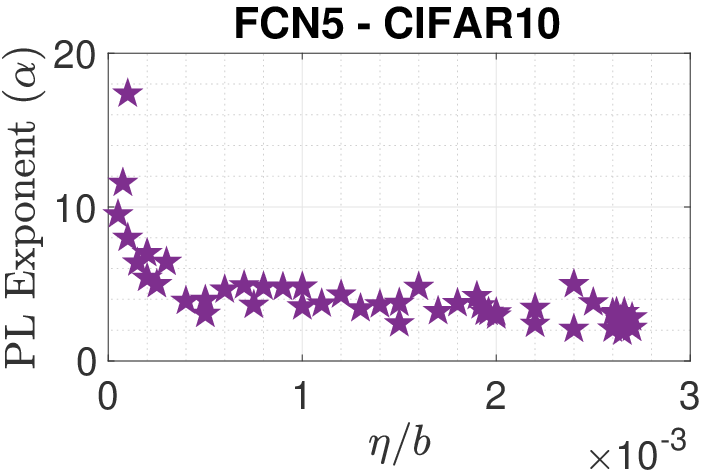}\\
\includegraphics[width=0.243\textwidth]{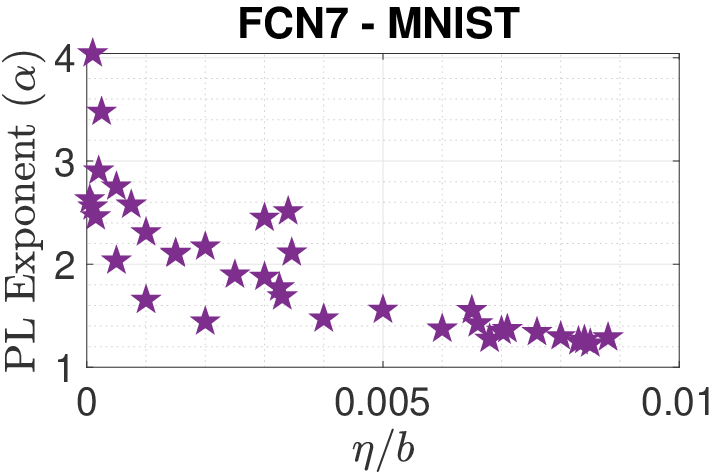}
\includegraphics[width=0.23\textwidth]{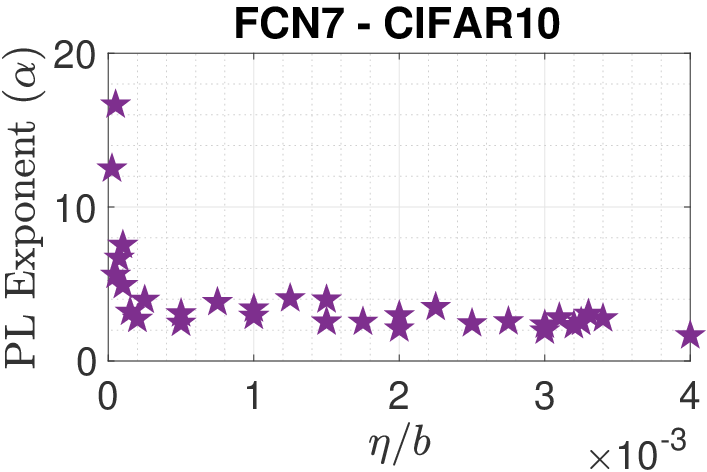}\\
\includegraphics[width=0.23\textwidth]{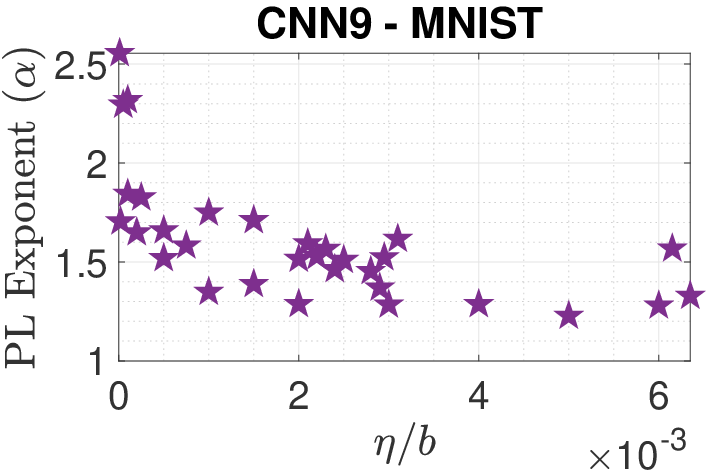} 
\includegraphics[width=0.23\textwidth]{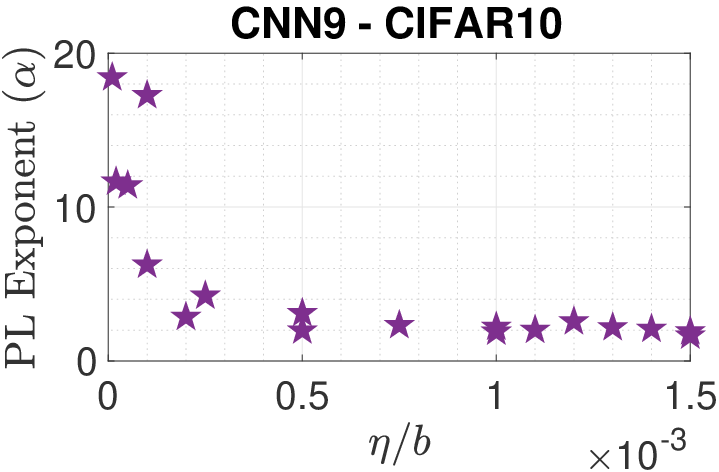}
\caption{Lower tail exponents versus step-size/batch-size ratio ($\eta/b$). We apply the same significance test as in Figure~\ref{fig:exps_pl}: the $p$-value ranges from  $10^{-9}$ to $10^{-3}$. \label{fig:exps_pl_etab}}
\vspace{-.5cm}
\end{figure}
\fi

It has been empirically demonstrated that the generalization performance can depend on the ratio of the step-size to the batch-size $\eta/b$ \citep{jastrzkebski2017three}.
Next, we monitor the behavior of the lower tail exponent $\alpha$ with respect to the SGD hyperparameters: step-size and batch-size.
Figure~\ref{fig:exps_pl_etab} shows the results. We observe that the local power-law behavior of SGD near the found local minimum also heavily depends on the ratio $\eta/b$, where we observe a clear monotonicity.
This reveals an interesting behavior that the hyperparameters $\eta$ and $b$ modify the local lower tail exponent of the transition kernel, hence the effective dimension, which in turn determines the generalization error. 
This outcome also shows interesting similarities with the recent studies of \citet{hodgkinson2020multiplicative, gurbuzbalaban2020heavy} that have shown that in online SGD (one pass regime with infinite data) the ratio $\eta/b$ determines the ``heaviness of the tails'' of the stationary distribution of SGD. Figure~\ref{fig:exps_pl_etab} suggests that in the finite training data setting (where \citet{hodgkinson2020multiplicative,gurbuzbalaban2020heavy} are not applicable), another type of power-law behavior is still observed in the local exponent $\alpha$, which also shows monotonic behavior with respect to $\eta/b$. We suspect that this monotonic behavior can be formally quantified, but we leave this as future work.

We finally note that estimating lower tail exponents accurately is a challenging task. For our estimator, it is known that larger fitted $\alpha$ values, e.g., greater than $\approx 5$, are less reliable. However, since our empirical results involve running the same experiment multiple times for different hyperparameters that are fairly close to each other, the consistency and the clear trends in our results support our theoretical contributions, even if the estimations are to be inexact. 

\subsection{Correlations between types of tail exponents}
\label{sec:StableExp}

Here, we shall discuss the relationship between the lower tail exponent in Corollary \ref{cor:BulkExponent} and the tail exponent seen in \citet{simsekli2019heavy,simsekli2020hausdorff,hodgkinson2020multiplicative,gurbuzbalaban2020heavy}. %
While the two are not related in general, in practice, we expect them to be somewhat correlated, and we justify this through the model considered in \citet{simsekli2019heavy}. 

Building on the perturbed GD model, \citet{simsekli2019tail} replaced the Gaussian updates by \emph{heavy-tailed} $\alpha$-stable distributed random variables, justified through the generalized central limit theorem. In this case, the model satisfies
$W_{k+1} = W_k - \eta (\nabla f(W_k) + S_k(W_k))$, where $S_k(w) = (S_k^i(w))_{i=1}^n$ and each $S_k^i(w)$ is independent symmetric $\alpha$-stable with scale $\sigma(w)$, that is, $S_k^i(w)$ has characteristic function $\varphi(t) = e^{-|\sigma(w) t|^\alpha}$. As $\eta \to 0^+$, this Markov chain behaves similarly to the stochastic differential equation
$\dd W_t = -\nabla f(W_t) \dd t + \eta^{\frac{\alpha-1}{\alpha}} \sigma \dd L_t^\alpha$,
where $L_t^\alpha$ is an $\alpha$-stable L\'{e}vy process (see \cite{simsekli2019heavy} for details). 
If $\nabla f$ is bounded, the process satisfies the following two properties \citep[Lemma~3]{bogdan2007estimates}:
(1) 
$\int_0^1 P_t(x, B_r(x)) \dd t~\sim~c_1 r^{\alpha}$ as $r~\to~0^+$; and (2)
$\int_0^1 P_t(x, B_r(x)) \dd t \sim c_2 r^{-\alpha}\quad \mbox{as } r \to \infty$. 
Therefore, the lower tail exponent in Corollary \ref{cor:BulkExponent} and the tail exponent in \cite{simsekli2019heavy} are identical in this model.

Empirically, we find that while the two exponents are not identical, they do appear to~correlate.
\ifdefined\nips
\begin{figure}[t]
\centering
\includegraphics[width=0.161\textwidth]{figures/5_00_fcn_mnist_pl_tailix.png}
\includegraphics[width=0.161\textwidth]{figures/5_00_fcn_cifar10_pl_tailix.png}
\includegraphics[width=0.161\textwidth]{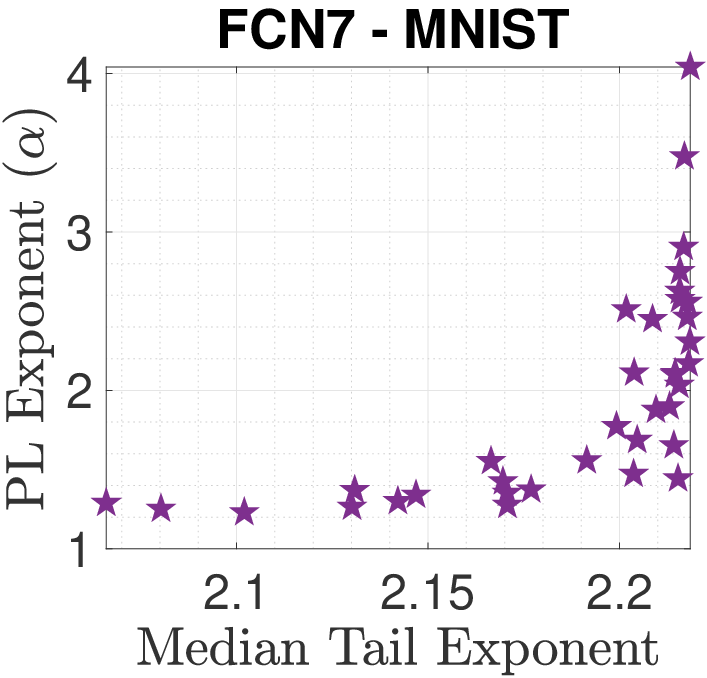}
\includegraphics[width=0.161\textwidth]{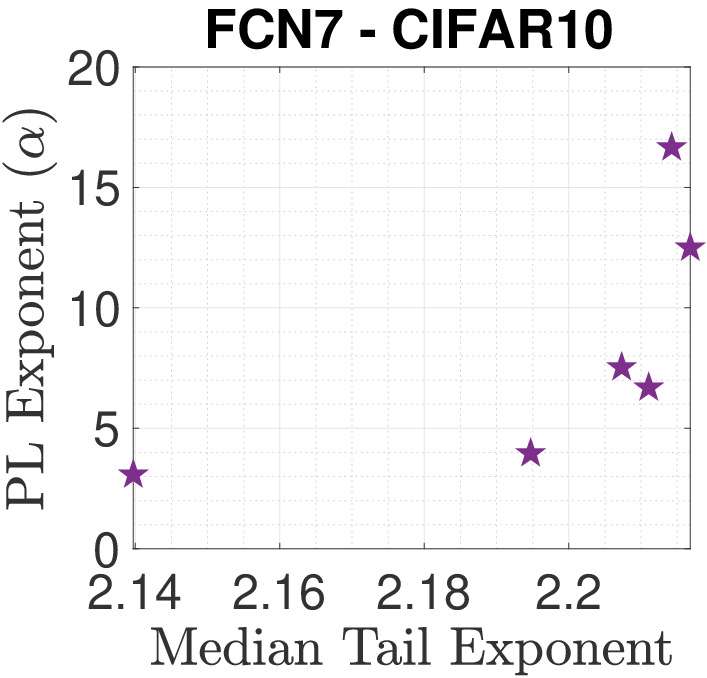}
\includegraphics[width=0.161\textwidth]{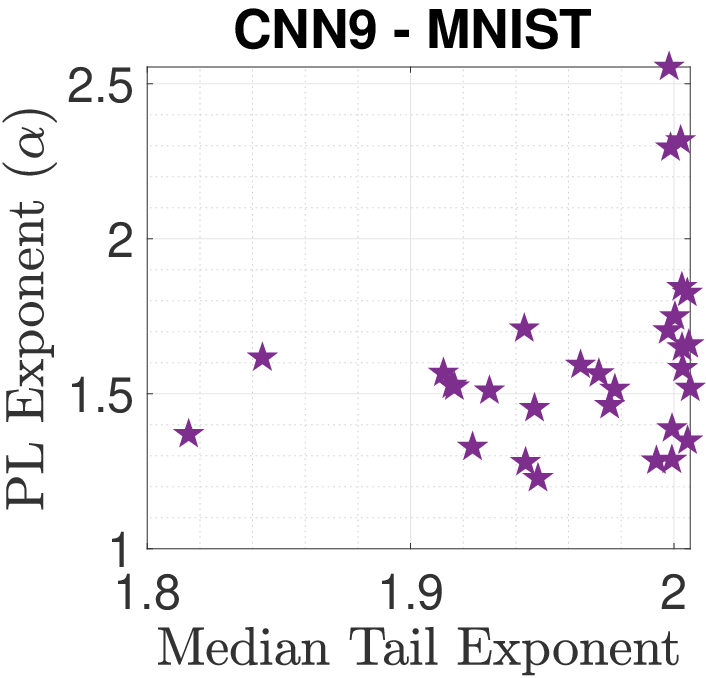}
\includegraphics[width=0.161\textwidth]{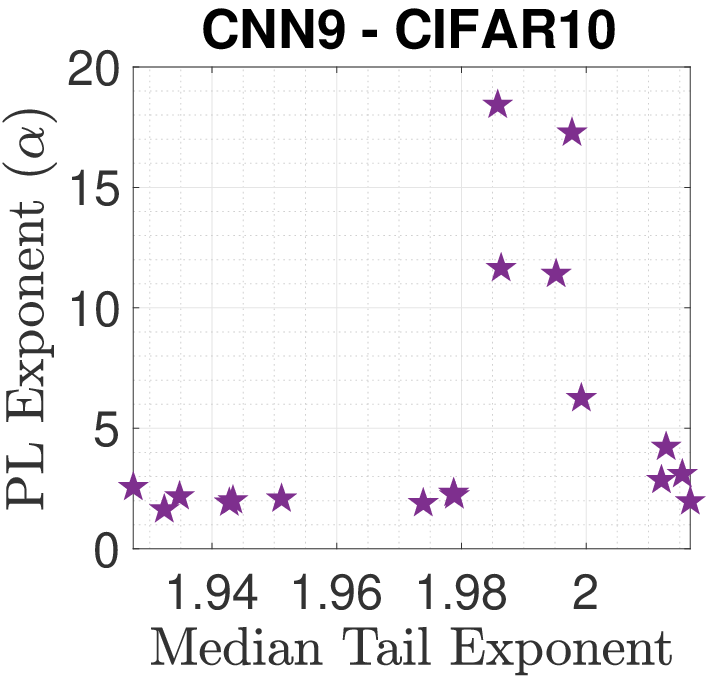}
\vspace{-10pt}
\caption{Lower tail exponents versus the tail-exponent estimates used in \cite{gurbuzbalaban2020heavy}.
\michael{Number figures consistently with main text.}
\michael{Make bigger, split over two lines, and deal with large versus small alpha scales consistently with main text.}
}
\label{fig:exps_pl_tailix}
\end{figure}
\else
\begin{figure}[t]
\centering
\includegraphics[width=0.23\textwidth]{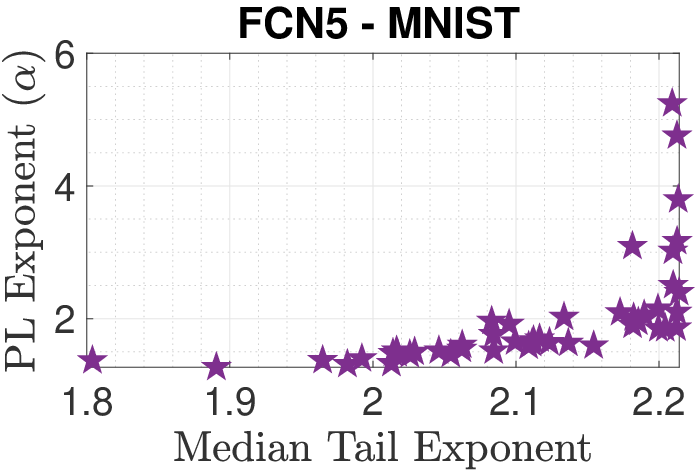}
\includegraphics[width=0.23\textwidth]{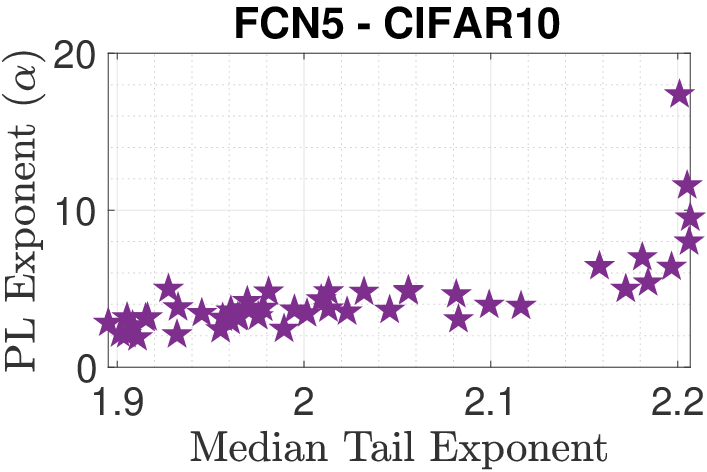}\\
\includegraphics[width=0.23\textwidth]{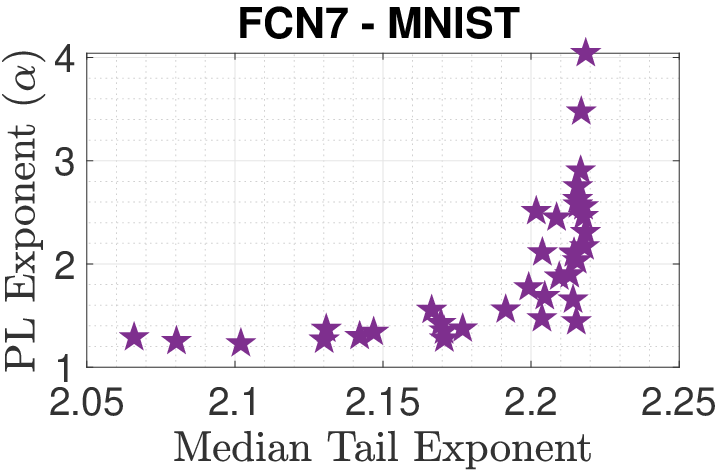} 
\includegraphics[width=0.23\textwidth]{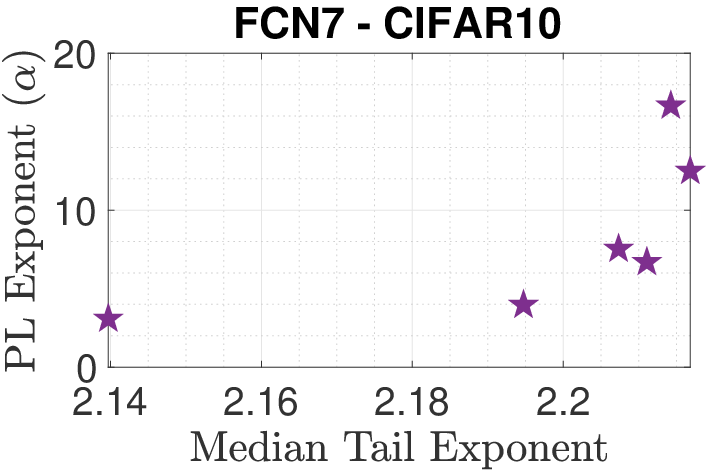}\\
\includegraphics[width=0.23\textwidth]{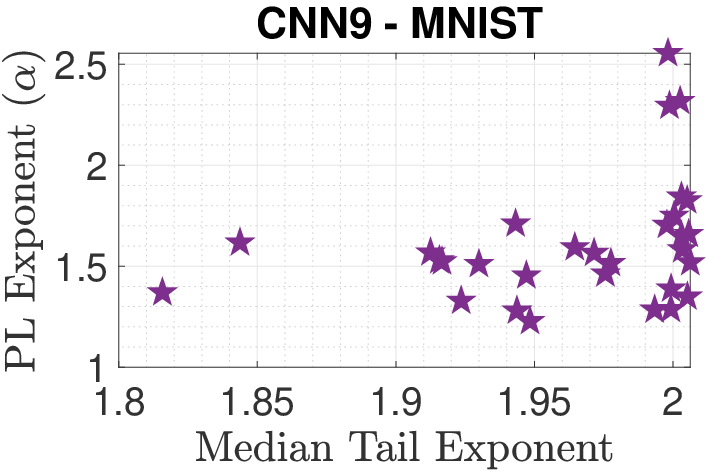}
\includegraphics[width=0.23\textwidth]{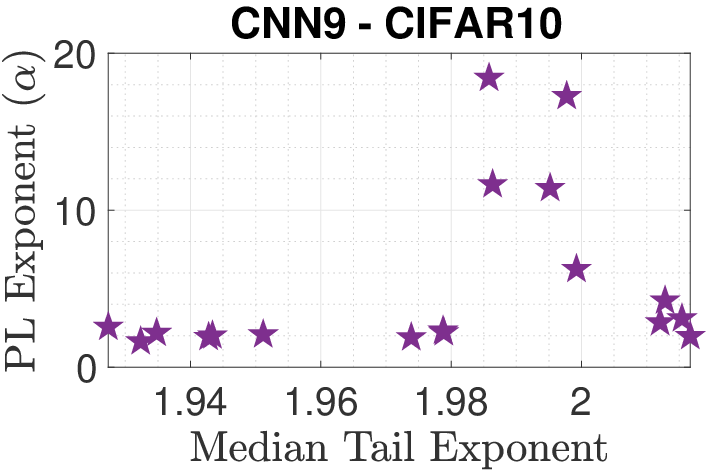}
\caption{Lower tail exponents versus the tail-exponent estimates used in \cite{gurbuzbalaban2020heavy}.}
\label{fig:exps_pl_tailix}
\vspace{-.5cm}
\end{figure}
\fi
In Figure~\ref{fig:exps_pl_tailix}, we follow the setup of \citet{gurbuzbalaban2020heavy} and assume that each layer of the neural network possesses a different tail-exponent, where these layer-wise tail-exponents are computed on averaged iterates under the assumption that they are distributed from an $\alpha$-stable distribution. We use the same tail-exponent estimator \cite{mohammadi2015estimating} as the one used in \citet{gurbuzbalaban2020heavy}.
Once the layer-wise indices are computed, we compute the median tail-exponent over the layers. As we can observe from Figure~\ref{fig:exps_pl_tailix}, our lower tail exponent and the tail-exponent estimate show an overall strong correlation. More precisely, when we estimate the $p$-values under a linear model, we observe that the $p$-value ranges from  $10^{-9}$ to $10^{-1}$, where $10^{-1}$ corresponds to CNN9-CIFAR10 result, where the correlation is not as strong.

These results shed more light on the empirical observations provided in \cite{simsekli2019tail}, as the estimator in \cite{mohammadi2015estimating} might be capturing the lower-tail behavior of the gradient noise, which seems to be the key factor according to our theory.

\ifdefined\nips
\vspace{-10pt}
\section{Conclusion}
\vspace{-10pt}
\else
\section{Conclusion}
\fi
\label{sec:conc}

We have developed a theoretical framework for analyzing the generalization properties of stochastic optimizers by using their trajectories.  We first proved generalization bounds based on the celebrated Fernique--Talagrand functional; and then, by using the Markovian structure of stochastic optimizers, we specialized these results to Markov processes. 
Using these results, we linked the accumulated generalization gap over an optimizer trajectory to a lower tail exponent in the transition kernel via a random walk approximation about a local minimum. This provides a discrete-time analogue of the work of \citet{simsekli2020hausdorff} to more practical stochastic optimizer models. 
Finally, we supported our theory with empirical results on several simple neural network models, finding correlations between the lower tail exponent, generalization gap at the end of training, step-size/batch-size ratio, and upper tail exponents. 

Our analysis raises a few unresolved questions. Firstly, while suggested by the theory, it is untested whether the correlation between generalization gap and the lower tail exponent is universal, or holds only due to correlations between the lower tail exponent and the hyperparameters of the optimizer in our setup (which are known to affect performance). To assess this, one would need to be able to alter the lower tail exponent without changing any other commonly considered hyperparameters (e.g., step-size/batch-size). Furthermore, we have only considered fixed step sizes, and it is unclear whether our theory can be extended to hold under common step size schedules. Finally, while the subgaussian assumption on the data leads to the typical $\mathcal{O}(n^{-1/2})$ rate in Theorem \ref{thm:NormFL}, it is known that this rate is not reflected in real-world settings \cite{kaplan2020scaling}. Different distributional assumptions on the data could yield more accurate rates. We leave these three problems to be addressed in future work.

\section*{Acknowledgments}

We would like to acknowledge DARPA, NSF, and ONR for providing partial support of this work. U.\c{S}.'s research is supported by the French government under management of Agence Nationale de la Recherche as part of the ``Investissements d'avenir'' program, reference ANR-19-P3IA-0001 (PRAIRIE 3IA Institute).

\ifdefined\nips
\section*{Checklist}

\begin{enumerate}

\item For all authors...
\begin{enumerate}
  \item Do the main claims made in the abstract and introduction accurately reflect the paper's contributions and scope?
    \answerYes{}
  \item Did you describe the limitations of your work?
    \answerYes{See Section \ref{sec:conc}.}
  \item Did you discuss any potential negative societal impacts of your work?
    \answerYes{See Section \ref{sec:conc}.}
  \item Have you read the ethics review guidelines and ensured that your paper conforms to them?
    \answerYes{}
\end{enumerate}

\item If you are including theoretical results...
\begin{enumerate}
  \item Did you state the full set of assumptions of all theoretical results?
    \answerYes{}
	\item Did you include complete proofs of all theoretical results?
    \answerYes{}
\end{enumerate}

\item If you ran experiments...
\begin{enumerate}
  \item Did you include the code, data, and instructions needed to reproduce the main experimental results (either in the supplemental material or as a URL)?
    \answerYes{The code is provided in the supplementary material.}
  \item Did you specify all the training details (e.g., data splits, hyperparameters, how they were chosen)?
    \answerYes{See Section \ref{sec:Experiments}}
	\item Did you report error bars (e.g., with respect to the random seed after running experiments multiple times)?
    \answerNA{}
	\item Did you include the total amount of compute and the type of resources used (e.g., type of GPUs, internal cluster, or cloud provider)?
    \answerYes{See Section \ref{sec:Experiments}}
\end{enumerate}

\item If you are using existing assets (e.g., code, data, models) or curating/releasing new assets...
\begin{enumerate}
  \item If your work uses existing assets, did you cite the creators?
    \answerNA{We have only used standard benchmarks.}
  \item Did you mention the license of the assets?
    \answerNA{}
  \item Did you include any new assets either in the supplemental material or as a URL?
    \answerNA{}
  \item Did you discuss whether and how consent was obtained from people whose data you're using/curating?
    \answerNA{}
  \item Did you discuss whether the data you are using/curating contains personally identifiable information or offensive content?
    \answerNA{}
\end{enumerate}

\item If you used crowdsourcing or conducted research with human subjects...
\begin{enumerate}
  \item Did you include the full text of instructions given to participants and screenshots, if applicable?
    \answerNA{}
  \item Did you describe any potential participant risks, with links to Institutional Review Board (IRB) approvals, if applicable?
    \answerNA{}
  \item Did you include the estimated hourly wage paid to participants and the total amount spent on participant compensation?
    \answerNA{}
\end{enumerate}

\end{enumerate}
\fi
\ifdef\includeappendix
{
\newpage

\onecolumn

\appendix

\section{Hausdorff dimension}
\label{sec:HausdorffApp}

Here, we shall review facts about the Hausdorff dimension.
The Hausdorff measure $\mathscr{H}^{m}$ on $\mathbb{R}^{d}$ is defined
on a set $E\subseteq\mathbb{R}^{d}$~by
\[
\mathscr{H}^{m}(E)=\lim_{\delta\to0^{+}}\inf_{\substack{E\subseteq\bigcup_{j}S_{j}\\
\diam(S_{j})\leq\delta
}
}\sum_{j}\alpha(m)2^{-m}\diam(S_{j})^{m},
\]
where the infimum is taken over countable covers $S_{\delta}$ of
$E$ by non-empty subsets of $\mathbb{R}^{d}$ with diameter not exceeding
$\delta$, and $\alpha(m) = \frac{\pi^{n/2}}{\Gamma(n/2+1)}$ is the volume of the unit $m$-sphere. The following facts are fundamental
\ifdefined\nips
\footnote{Herbert Federer. Geometric Measure Theory. Springer, 2014. \S2.10.2, \S2.10.35}
\else
\citep[\S2.10.2, \S2.10.35]{federer2014geometric}
\fi
: (1) if $\mathscr{H}^{m}(E)<\infty$,
then $\mathscr{H}^{l}(E)=0$ for any $l>m$; %
and (2) on $\mathbb{R}^{d}$, $\mathscr{H}^{d}\equiv\mathscr{L}^{d}$
the $d$-dimensional Lebesgue~measure. %

The first fact implies the existence of the \emph{Hausdorff dimension,
}which is defined for a set $E\subseteq\mathbb{R}^{d}$ by $\dimh(E)=\inf\{m:\mathscr{H}^{m}(E)=0\}$.
The second fact implies that any set $E\subset\mathbb{R}^{d}$ with
$\dimh(E)<d$ has zero Lebesgue measure; but, importantly, the converse
is not true. In differential geometry, the Hausdorff dimension is
useful for identifying the dimension of submanifolds. However, it
is also of significant value in the study of fractal sets, providing
a measurement of ``clustering'' in space. One can intuit this from
the definition, but it is perhaps best seen through examples. In Figure~\ref{fig:BrownLevy}, a a L\'{e}vy process with sample paths possessing Hausdorff dimension $1.5$ is compared to Brownian motion, whose sample paths have Hausdorff dimension $2$. The L\'{e}vy process, with the smaller Hausdorff dimension, exhibits dynamics with tighter clusters separated by large jumps.

Ahlfors lower-regularity plays a key role in Corollary \ref{cor:BulkExponent}. A more commonly considered definition is Ahlfors regularity itself. 
\begin{definition}[Ahlfors regular]
A set $W \subset \mathbb{R}^D$ is $\alpha$-Ahlfors regular if there exists a measure $\mu$ on $W$ and $c_1,c_2,r_0 > 0$ such that
\[
c_1 r^\alpha \leq \mu(B_r(w)) \leq c_2 r^\alpha,\qquad \mbox{for all }w \in W,\, 0 < r < r_0.
\]
If only the lower (upper) inequalities are satisfied, $W$ is said to be $\alpha$-Ahlfors lower-regular (upper-regular). 
\end{definition}
Ahlfors regularity is often assumed to equate several notions of fractal dimension, such as in \cite{simsekli2020hausdorff}. For example, the lower and upper Minkowski dimensions of a set $W$ are given by
\[
\underline{\dim}_M(W) = \liminf_{ r \to 0^+} \frac{\log N_r(W)}{|\log r|},\quad \overline{\dim}_M(W) = \liminf_{ r \to 0^+} \frac{\log N_r(W)}{|\log r|},
\]
respectively, where $N_r(W)$ denotes the smallest number of balls of radius $r$ needed to cover $W$. The following holds.
\begin{theorem*}[\cite{mattila1999geometry}, Theorem 5.7]
For any $\alpha$-Ahlfors regular set $W$, $\underline{\dim}_{\rm M}(W) = \overline{\dim}_{\rm M}(W) = \dimh(W) = \alpha$. 
\end{theorem*}
To link Hausdorff dimension to continuous-time optimization, we rely on \citet[Theorem 4.2]{xiao2003random}, restated below. 
\begin{theorem*}[\cite{xiao2003random}, Theorem 4.2]
Let $X_t$ be a continuous-time Markov process in $\mathbb{R}^D$ with transition kernel $P_t(x,\dd y)$ satisfying $P_t(x,B_r(x)) \geq K$ for sufficiently large $r > 0$ and $c^{-1} P_t(0,B_r(0)) \leq P_t(x,B_r(x)) \leq c P_t(0,B_r(0))$ for all $r > 0$ and some $c > 0$. Then the Hausdorff dimension of $\{X_t\}_{t\in[0,1]}$ is 
\begin{align*}
\alpha &= \sup\left\{\gamma \geq 0 \,:\, \lim_{r \to 0^+} r^{-\gamma} \int_0^1 P_t(0,B_r(0)) \dd t < +\infty\right\}\\
&= \sup\left\{\gamma \geq 0 \,:\, \int_0^1 \mathbb{E}(|X_t|^{-\gamma}) \dd t < +\infty\right\}.
\end{align*}
\end{theorem*}

\section{Variance and clustering in the Fernique--Talagrand functional}
\label{sec:VarCluster}
Now that two generalization bounds in Theorem \ref{thm:bound_transition_kernel} and Corollary \ref{cor:BulkExponent} have been obtained involving the dynamics of the stochastic optimizer, we shall briefly discuss how two properties of the trajectory --- variance and clustering --- play a critical role in the Fernique--Talagrand functional. 

First, we discuss the influence of ``variance,'' or the average size of fluctuations of the stochastic optimizer. Drawing from Theorem \ref{thm:bound_transition_kernel}, we consider a continuous-time stochastic optimization model in the form of a stochastic differential equation
$
\dd W_t = \mu(W_t) \dd t + \sigma(W_t) \dd B_t,
$ 
where $\mu$ typically involves the gradient of the empirical risk $\mathcal{R}_n$. We assume that $\mu,\sigma$ are bounded with bounded second derivatives, and $\sigma(x)$ has non-zero singular values for all $x \in \mathbb{R}^D$. The Aronson estimates \cite{sheu1991some} imply that for some constant $K > 0$ and ``minimal variance'' $\lambda > 0$, $P_t(x,E) \geq K t^{-D/2} \int_E \exp(- \|x - y\|^2  / (2\lambda t)) \dd y$. 
This, together with Theorem \ref{thm:bound_transition_kernel}, implies the existence of constants $K_1,K_2 > 0$ independent of $T,\lambda$ such that 
(see Appendix \ref{sec:GrowthExp})
\begin{equation}
\label{eq:SDE_FT}
\mathbb{E}\gamma_2^\rho(\{W_t\}_{t\in[0,T]}) \leq \frac{K_1}{\rho} \int_0^\rho \sqrt{K_2 - D \log r - \log J_{\rho,T}(\tfrac{1}{2\lambda}, D)} \dd r,
\end{equation}
where $J_{\rho,T}(a,D) = \frac{1}{T}\int_{0}^{1}\int_{1/T}^{\infty}v^{D/2-1}s^{D/2-2}e^{-asv\rho^{2}}\dd s \dd v$ is monotone decreasing in $a$ and monotone increasing in $D$. It is important to note that $\lambda$ will generally increase with the magnitude of $\mu,\sigma$, and hence with the size of the stochastic gradient. Therefore, the normalized FT functional should increase monotonically with the variance of the fluctuations of the optimizer. Consequently, the \emph{sharpness} of an optimum should be reflected in $\gamma_2^\rho$: in a ``flat'' neighborhood, the variance of fluctuations in the optimizer is small, and hence $\gamma_2^\rho$ will be smaller. %

Now, we move on to discuss how the $\alpha$ parameter in Corollary \ref{cor:BulkExponent} can be interpreted as a measure of ``clustering'' of the optimizer trajectory. To accomplish this, we invoke the $K$-function of  \citet{ripley1976second}, a commonly used spatial statistic to determine spatial inhomogeneity. For a point process $N$ on $\mathbb{R}^D$, the $K$-function is defined for each $r > 0$ as the ratio of the expected number of points within distance $r$ of any randomly chosen point, and the average density of points. Letting $W = \{w_1,\dots,w_n\}$ denote a realization of a point process, its $K$-function can be estimated consistently by \citet[eqn. 8.2.18]{cressie2015statistics}:
\begin{equation}
\label{eq:KFunction}
\hat{K}(r) = \frac{\diam(W)}{n} \sum_{\substack{i,j = 1 \\ i \neq j}}^n \ind\{\|w_i - w_j\| \leq r\},\qquad r > 0.
\end{equation}
The $K$-function of a homogeneous Poisson process on $\mathbb{R}^D$ with constant intensity $\lambda$ is given by 
$
K(r) = \pi^{D/2} r^D / \Gamma(1+\frac{D}{2})
$
 for $r > 0$ \citep[eqn. 8.3.34]{cressie2015statistics}. Note that this function is independent of the intensity. 
Therefore, deviations in the $K$-function from an $\mathcal{O}(r^D)$ growth rate suggest spatial inhomogeneity, and therefore points which exhibit spatial clustering.
Treating an optimizer trajectory $W_1,\dots,W_n$ as the realization of a point process, we can interpret the $K$-function through (\ref{eq:KFunction}) and measure spatial clustering of the trajectory. Following the assumptions in Corollary \ref{cor:BulkExponent} which define $\alpha$, we can assume that
$\|W_i - W_j\| \leq r$ with probability approximately $c_{|i-j|} r^{\alpha}$ for $0 < r < r_0$, where $c_1,\dots,c_n$ are positive constants. Under this assumption, $\hat{K}(r) \approx C \diam(W) r^{\alpha}$ for $0 < r < r_0$ and some constant $C > 0$. Therefore, we would expect $\alpha \ll D$ to be an indication that the optimization exhibits significant spatial clustering. By Corollary \ref{cor:BulkExponent}, this circumstance should coincide with a smaller Fernique--Talagrand functional, and therefore improved generalization. 

\begin{figure}
    \centering
    \hspace{-1cm}\includegraphics[width=0.7\textwidth]{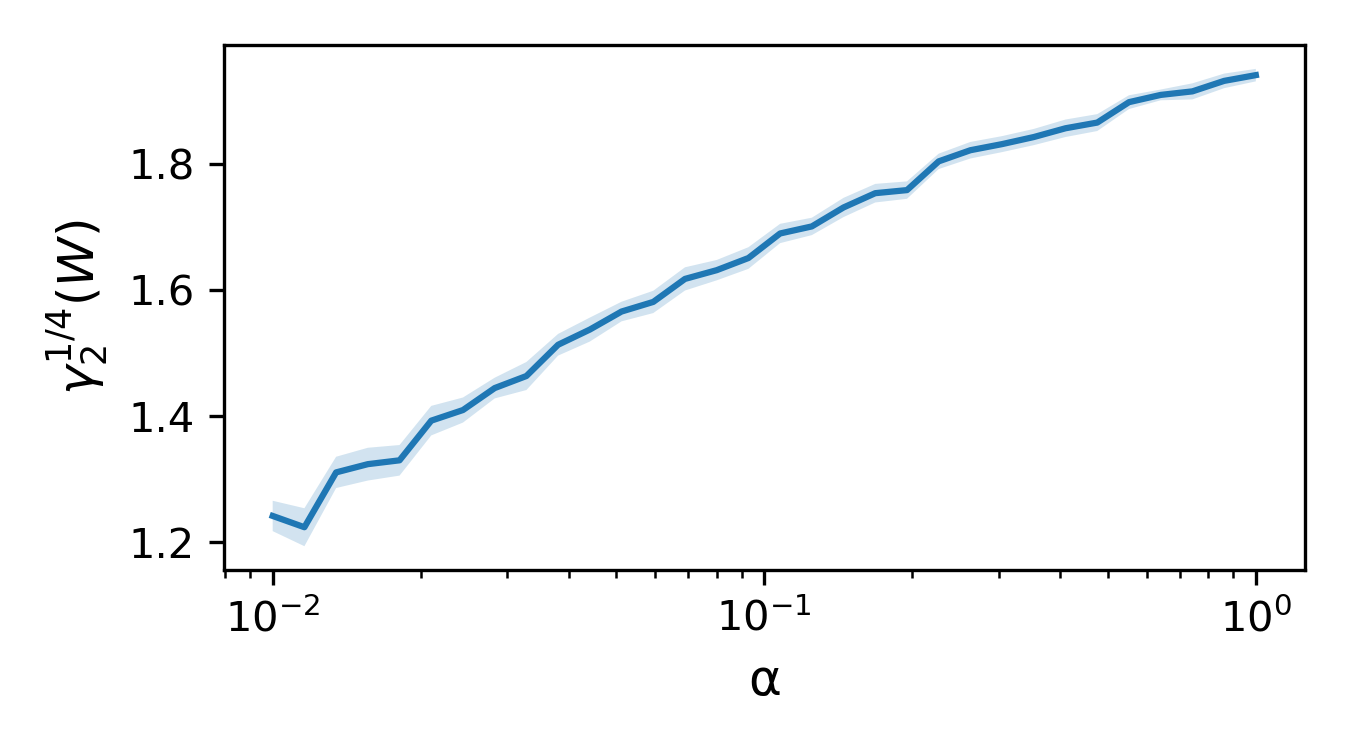}
    \caption{Lower tail exponents versus normalized Fernique--Talagrand functional $\gamma_2^{1/4}(W)$ for $\rho = 1/4$, averaged over 100 runs with 95\% confidence intervals shaded.}
    \label{fig:FTvsAlpha}
\end{figure}

\section{Lower tail exponent and the Fernique--Talagrand functional}
\label{sec:FTAlpha}

To further demonstrate the relationship between the lower tail exponent of a Markov transition kernel and the Fernique--Talagrand functional over its sample path, we consider an extension of the setup in Figure \ref{fig:BrownLevy} and measure $\gamma_2^\rho(\{W_k\}_{k=0}^m)$ for random walks with prescribed lower tail exponents. 

Consider a Markov chain $W_k$ on $\mathbb{R}^2$ defined as follows: starting from $W_0$, let 
\[
W_{k+1} = W_k + (\cos(U_k), \sin(U_k)) Z_k,
\]
where for each $k=0,1,\dots$, $U_k \sim \mathcal{U}(-\pi,\pi)$, and $Z_k \sim \beta'(\alpha,\beta)$ are independent, and $\beta'(\alpha,\beta)$ is the beta prime distribution with density
\[
p(x) = \begin{cases}
\frac{\Gamma(\alpha+\beta)}{\Gamma(\alpha)\Gamma(\beta)} x^{\alpha - 1} (1 + x)^{-\alpha-\beta}, &\qquad \mbox{ if } x > 0,\\
0,&\qquad \mbox{ otherwise.}
\end{cases}
\]
The lower tail exponent of this process, as defined in Section \ref{sec:LowerTail}, is precisely $\alpha$. To remove the effect of variance, the chain is normalized as $\bar{W}_k = W_k / S(\{W_l\}_{l=0}^k)$, where $S(W)$ is the coordinate-wise standard deviation of $W$. For fixed $W_0 = 0$, $\beta = 3.5$ and varying $\alpha \in [10^{-2},1]$, we plot the normalized Fernique--Talagrand functional $\gamma_2^{1/4}(\{\bar{W}_k\}_{k=0}^{100})$ (with $\rho = 1/4$) of the first $100$ iterates of $W_k$, averaged over 100 runs. 
The result is shown in Figure \ref{fig:FTvsAlpha}. As expected, the FT functional grows with $\alpha$. However, unlike the behaviour suggested in Corollary \ref{cor:BulkExponent}, the FT functional appears to grow like $\log \alpha$ for small $\alpha$. 

\section{Lower tail exponent as intrinsic dimension}
\label{sec:GrowthExp}

Here, we shall discuss the relationship between the exponent $\alpha$ and the dimension $D$.
The most commonly considered discrete-time model for a stochastic optimizer is the perturbed gradient descent (GD) model, which satisfies
$W_{k+1} = W_k - \gamma (\nabla f(W_k) + Z_k)$, where $Z_k$ is a Gaussian random vector with zero mean and constant covariance matrix $\Sigma$. In the neighborhood of a local minimum $w^\ast$, $\nabla f(W_k) \approx 0$, and hence the perturbed gradient descent model resembles a Gaussian random walk $W_{k+1} = W_k + \gamma Z_k$. In this case, if $D$ denotes the ambient dimension, the $k$-step transition kernel becomes that of a $D$-dimensional multivariate normal distribution with covariance matrix $k \Sigma$:
\[
P^k(x, E) = (2\pi k)^{-\frac{D}{2}} |\det \Sigma|^{-\frac12} \int_E \exp\left(-\tfrac1{2k} (y-x)^\top \Sigma^{-1} (y-x)\right) \dd y.
\]
Since $\Sigma$ is necessarily positive-definite, letting $\sigma_1$ and $\sigma_D$ denote the largest and smallest singular values of $\Sigma$ respectively, for any $x \in \mathbb{R}^D$ and set $E \subset \mathbb{R}^D$,
\begin{equation}
\label{eq:GaussianTransBound}
\frac{1}{(2\pi k \sigma_1)^{D/2}} \int_E \exp\left(-\frac{\|y-x\|^2}{2 k \sigma_D}\right) \dd y \leq P^k(x, E) \leq \frac{1}{(2\pi k \sigma_D)^{D/2}} \int_E \exp\left(-\frac{\|y-x\|^2}{2 k \sigma_1}\right) \dd y.
\end{equation}
Applying Lemma \ref{lem:GaussInt} to (\ref{eq:GaussianTransBound}), we see that for the perturbed GD model, $\alpha$ in Corollary \ref{cor:BulkExponent} is precisely the ambient dimension $D$. 

\section{Direct estimation of the Fernique--Talagrand functional}

An attractive feature of the FT functional is the availability of a low-degree polynomial time approximation algorithm for $\gamma_2(W,d)$ when $W$ is a finite subset of $\mathbb{R}^D$. In particular, \cite{borst2020majorizing} shows that $\gamma_2(W,d)$ is computable to $\epsilon$-accuracy in $\mathcal{O}((|W|^{1+\omega} + D|W|^3) \log(|W| D /\epsilon))$ time, and $\omega \leq 2.373$ is the matrix multiplication exponent. To our knowledge, this functional is the only object which tightly bounds (\ref{eq:AccError}), up to constants, and is approximable in polynomial time. 
This is especially attractive for our purposes, providing an effective measure of generalization performance and exploratory capacity, which does \emph{not} require access to any test data.
A general procedure to estimate $\gamma_2(W,d)$ is presented in Algorithm \ref{alg:FTEst}. To perform the optimization in the final step, any off-the-shelf nonlinear optimization procedure will suffice, including gradient descent. Indeed, since the objective $I_W(p)$ happens to be convex in $p$ (see \cite{borst2020majorizing}), any local minimizer $p^\ast$ with subgradient $\partial I_W(p^\ast) = 0$ will satisfy $\gamma_2(W, d) = I_W(p^\ast)$.

\ifdefined\nips
\begin{minipage}{.48\textwidth}
    \includegraphics[width=\textwidth]{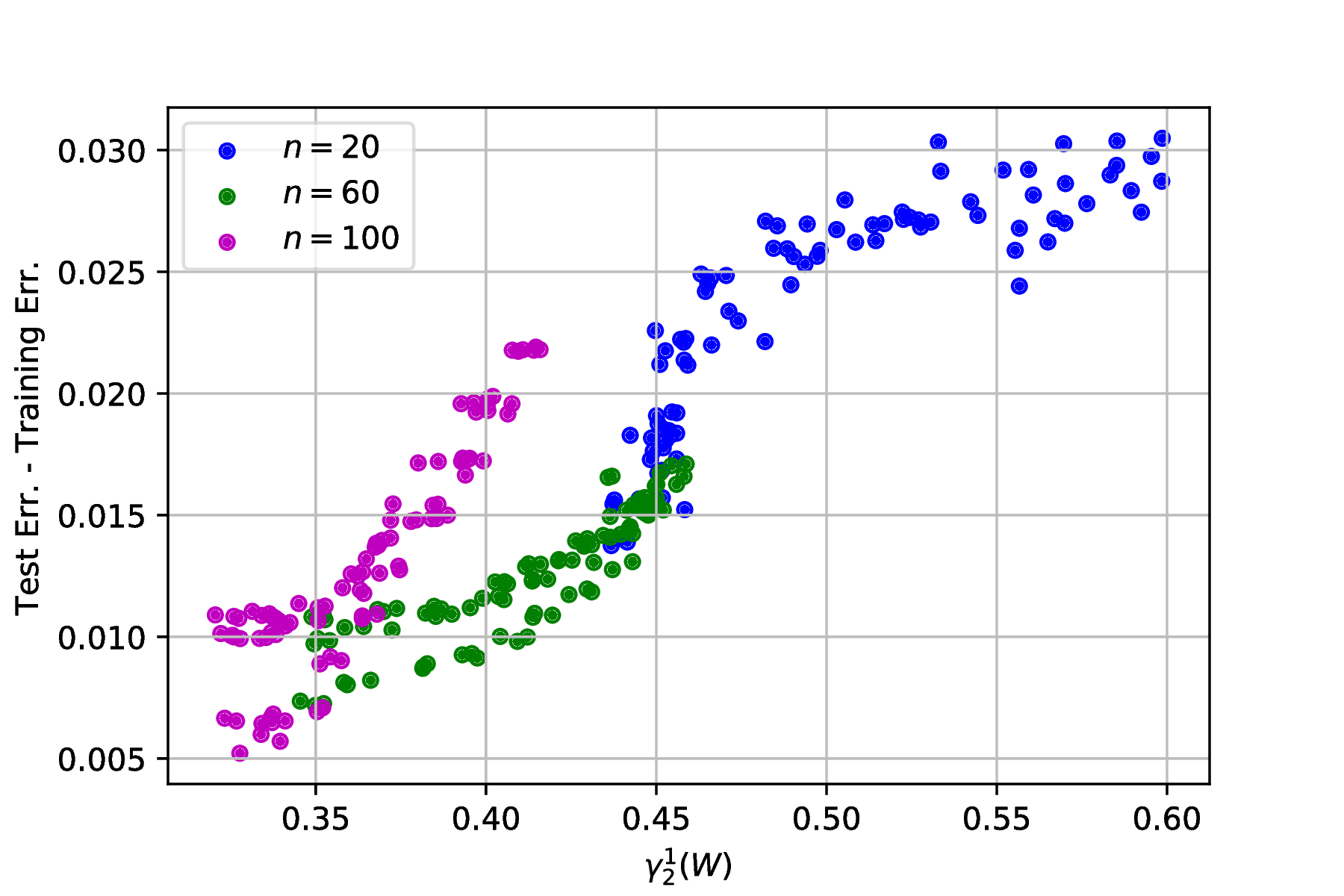}
\captionof{figure}{Normalized Fernique--Talagrand functional $\gamma_2^1(W)$ versus generalization gap. Different colors denote different batch-sizes.
\label{fig:exps_ft1}}
\end{minipage}
\hspace{.03\textwidth}
\begin{minipage}{.48\textwidth}
\begin{algorithm}[H]
   \caption{Fernique--Talagrand functional}
   \begin{algorithmic}
   \small\STATE {\bfseries Input:} A set $W = \{w_1,\dots,w_n\}$, and metric $d(w,w')$ on $W$
   \STATE {\bfseries Output:} An estimate of $\gamma_2(W,d)$.
   \end{algorithmic}
   
    \begin{algorithmic}[1]
    \small\STATE Compute Gram matrix $G = (d(w_i,w_j))_{i,j=1}^n$
    \smallskip
    \STATE {\bfseries for} $i=1,\dots,n$
    \STATE \hspace{1em}\begin{minipage}{0.8\textwidth}sort $(G_{ij})_{j=1}^n$ in ascending order to obtain $(\tilde{G}_{ij})_{j=1}^n$ and sorted indices $(\iota_{ij})_{j=1}^n$
    \end{minipage}
    \STATE {\bfseries end for}
    \smallskip
   \STATE {\bfseries function} $I_W(p)$
    \STATE \hspace{1em}{\bfseries for }$i=1,\dots,n$
    \STATE \hspace{1em} $I_W^i(p) = \sum_{j=0}^{n-1} \tilde{G}_{ij}\sqrt{|\log\sum_{k=0}^j p_{\iota_{ik}}|}$
    \STATE \hspace{1em}{\bfseries return} $\max\{I_W^1(p),\dots,I_W^n(p)\}$ 
   \STATE{\bfseries end function}
   \smallskip
   \STATE {\bfseries return} $\min_{z \in \mathbb{R}^n} I_W(\mbox{\textsf{softmax}}(z))$
\end{algorithmic}
\label{alg:FTEst}
\end{algorithm}
\end{minipage}
\else
\begin{figure}
\centering
    \includegraphics[width=0.75\columnwidth]{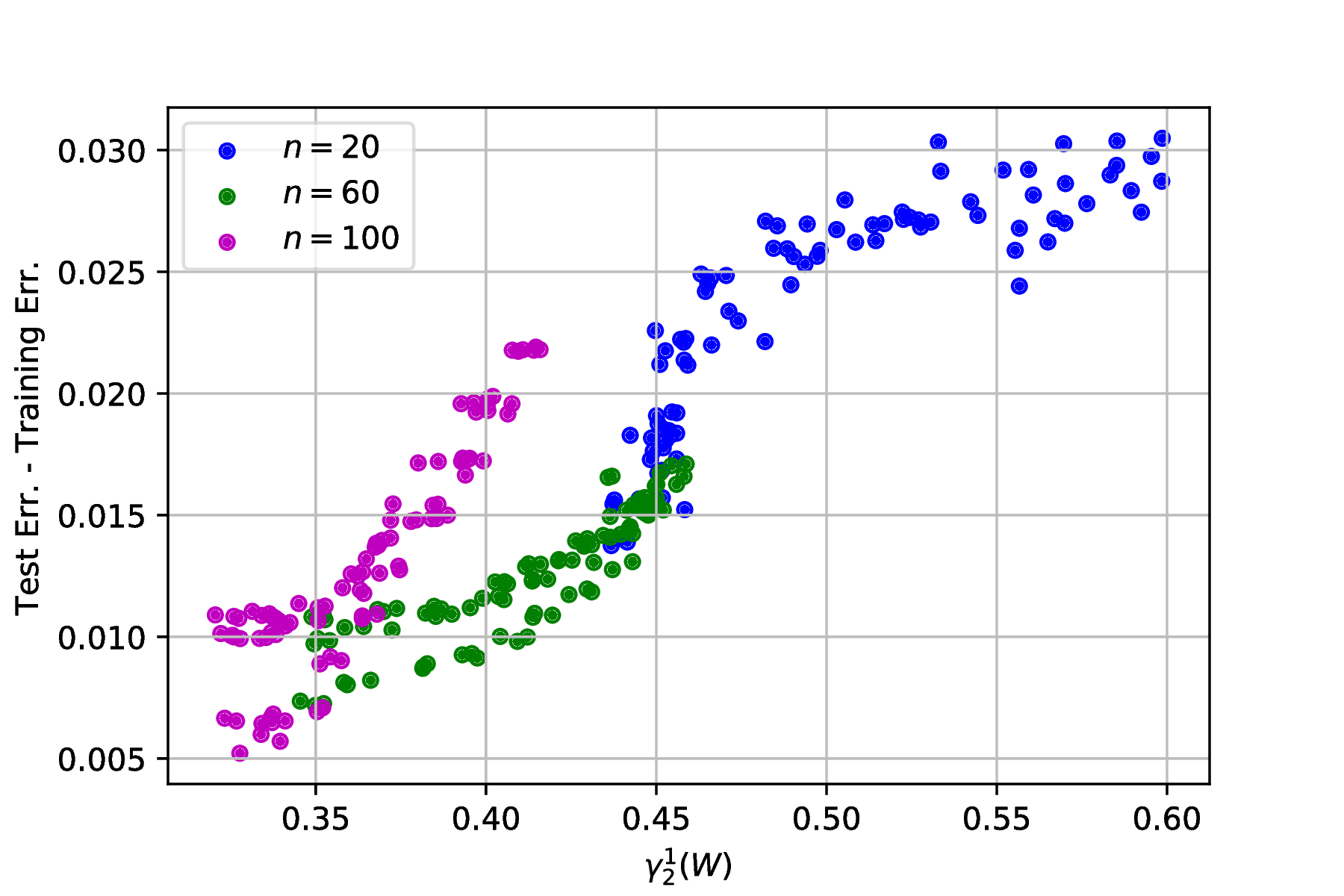}
\caption{Normalized Fernique--Talagrand functional $\gamma_2^1(W)$ versus generalization gap. Different colors denote different batch-sizes.\label{fig:exps_ft1}}
\end{figure}
\begin{algorithm}[H]
   \caption{Fernique--Talagrand functional}
   \begin{algorithmic}
   \small\STATE {\bfseries Input:} A set $W = \{w_1,\dots,w_n\}$, and metric $d(w,w')$ on $W$
   \STATE {\bfseries Output:} An estimate of $\gamma_2(W,d)$.
   \end{algorithmic}
   
    \begin{algorithmic}[1]
    \small\STATE Compute Gram matrix $G = (d(w_i,w_j))_{i,j=1}^n$
    \smallskip
    \STATE {\bfseries for} $i=1,\dots,n$
    \STATE \hspace{1em}\begin{minipage}{0.8\textwidth}sort $(G_{ij})_{j=1}^n$ in ascending order to obtain $(\tilde{G}_{ij})_{j=1}^n$ and\\ sorted indices $(\iota_{ij})_{j=1}^n$
    \end{minipage}
    \STATE {\bfseries end for}
    \smallskip
   \STATE {\bfseries function} $I_W(p)$
    \STATE \hspace{1em}{\bfseries for }$i=1,\dots,n$
    \STATE \hspace{1em} $I_W^i(p) = \sum_{j=0}^{n-1} ({G}_{i,j+1}-G_{ij})\sqrt{|\log\sum_{k=0}^j p_{\iota_{ik}}|}$
    \STATE \hspace{1em}{\bfseries return} $\max\{I_W^1(p),\dots,I_W^n(p)\}$ 
   \STATE{\bfseries end function}
   \smallskip
   \STATE {\bfseries return} $\min_{z \in \mathbb{R}^n} I_W(\mbox{\textsf{softmax}}(z))$
\end{algorithmic}
\label{alg:FTEst}
\end{algorithm}
\fi

Unfortunately, this approach becomes more challenging in high dimensions, where the optimization task becomes more difficult to solve to reasonable accuracy. Hence, we restrict our discussion in this section to smaller models. Our model of choice is a three-layer fully-connected neural network with 20 hidden units, applied to the least-squares regression task on the Wine Quality UCI dataset \cite{cortez2009modeling}. Models are trained from the same (random) initialization for 30 epochs (\emph{before} reaching 100\% training accuracy) using SGD with constant step size $\eta \in \{0.01,0.005,0.001\}$, batch size $b \in \{20,50,100\}$, weight decay parameter $\lambda \in [10^{-4},5\times 10^{-4},10^{-3}]$, and added zero-mean Gaussian noise to the input data with variance $\sigma^2$ for $\sigma \in \{0,0.05,0.1\}$. In Figure \ref{fig:exps_ft1}, for each model, we plot test error at the end of training against the estimated normalized FT functional (using Algorithm \ref{alg:FTEst}) from the last 50 iterates of training. The most profound difference in trends is seen with varying batch size. Nevertheless, as expected, the FT functional shows a strong correlation to test error. 

\section{Proofs of our main results}

\subsection{Proof of Theorem \ref{thm:NormFL} and Corollary \ref{cor:Hausdorff}}

The total mutual information is valuable as it precisely defines the degree to which we may decouple two random elements, as shown in the following lemma.

\begin{lemma}
For any Borel set $B$, $\mathbb{P}_{X,Y}(B)\leq\exp I_{\infty}(X,Y)\cdot\mathbb{P}_{X} \otimes \mathbb{P}_Y(B)$.
\end{lemma}
\begin{proof}
The proof relies on the data processing inequality for $\alpha$-Renyi divergence, which implies that for any $\alpha > 1$:
\[
D_{\alpha}(\text{B}(\mathbb{P}_{X,Y}(B))\Vert\text{B}(\mathbb{P}_{X} \otimes \mathbb{P}_Y (B)))\leq I_{\alpha}(X,Y),
\]
where $\text{B}(p)$ is a Bernoulli measure with success probability $p$. Therefore, letting $p = \mathbb{P}_{X,Y}(B)$ and $q = \mathbb{P}_X \otimes \mathbb{P}_Y (B)$, as $\alpha \to \infty$,
\begin{align*}
D_{\alpha}(\text{B}(p)\Vert\text{B}(q))
&=\frac{1}{\alpha-1}\log\left(q\frac{p^{\alpha}}{q^{\alpha}}+(1-q)\frac{(1-p)^{\alpha}}{(1-q)^{\alpha}}\right)\\
&=\log\left(\left(q\frac{p^{\alpha}}{q^{\alpha}}+(1-q)\frac{(1-p)^{\alpha}}{(1-q)^{\alpha}}\right)^{1/(\alpha-1)}\right)
\to\log\max\left\{ \frac{p}{q},\frac{1-p}{1-q}\right\}.
\end{align*}
Therefore, it follows that $p / q \leq \sup_{\alpha} \exp I_\alpha(X,Y) = \exp I_\infty(X,Y)$. 
\end{proof}

\begin{proof}[Proof of Theorem \ref{thm:NormFL}]
In the sequel, we shall let $K > 0$ denote a universal constant, not necessarily the same in each appearance. Let $d_\rho(w,w') = \min\{\rho, \|x - y\|\}$, and consider the alternative \emph{generic chaining} functional $\hat{\gamma}$ given by
\[
\hat{\gamma}_2(W, d_\rho) = \inf \sup_{w \in W} \sum_{k=1}^\infty 2^{n/2} d_\rho(w, T_k),
\]
where the infimum is taken over all sequences of subsets $\{T_k\}_{k=1}^{\infty}$ such that $|T_k| \leq N_k$ (where $N_0 = 1$ and $N_k = 2^{2^k}$ otherwise). By \citet[Theorem 1.1]{talagrand2001majorizing}, there exists a universal constant $K > 0$ such that $\hat{\gamma}_2(W, d_\rho) \leq K \gamma_2^\rho(W)$. 
The proof proceeds in a similar fashion to \citet[Theorem 2.2.27]{talagrand2014upper}.
For each $k$, let $T_k^W$ be a set such that $|T_k^W| \leq N_k$ and
\[
\sup_{w \in W} \sum_{k=1}^\infty 2^{k/2} d_\rho(w, T_k^W) \leq 2 \hat{\gamma}_2(W, d_\rho).
\]
To construct an increasing sequence of subsets, let $U_k^W = \bigcup_{m\leq k} T_m^W$, so that $U_0^W = T_0^W$ and $|U_k^W| \leq 2 N_k$. Now, let $Y_w = n^{1/2} [\mathcal{R}_n(w) - \mathcal{R}(w)] / \max\{B/\rho, L\}$, so that $\mathbb{E}Y_w = 0$, and by Hoeffding's inequality, for any $u > 0$ and $w,w'\in \mathbb{R}^D$,
\begin{equation}
\label{eq:ChainingSubg}
\mathbb{P}(|Y_w - Y_{w'}| > u d_\rho(w, w')) \leq 2 \exp(-\tfrac12 u^2).
\end{equation}
For $u > 0$, consider the event $\Omega(u)$ where
\[
\forall k \geq 1, \, \forall w,w' \in U_k^W,\, |Y_w - Y_{w'}| \leq 2 (2^{k/2} + u) d_\rho(w,w').
\]
For each $k$, let $\tilde{U}_k^W$ be an independent copy of $U_k^W$. Then, for $M = \exp I_\infty(X, W)$, 
\begin{align*}
\mathbb{P}(\Omega \backslash \Omega(u)) &\leq \sum_{k=1}^{\infty} \mathbb{E}\mathbb{P}(|Y_w - Y_{w'}| \leq 2(2^{k/2} + u)d_\rho(w,w'), \forall w,w' \in U_k^W \, \vert \, U_k^W)\\
&\leq M \sum_{k=1}^{\infty} \mathbb{E}\mathbb{P}(|Y_w - Y_{w'}| \leq 2(2^{k/2} + u)d_\rho(w,w'), \forall w,w' \in \tilde{U}_k^W \, \vert \, \tilde{U}_k^W)\\
&\leq M \sum_{k=1}^\infty 2^{2(2^k + 1)+1} \exp(-2 (2^k + u^2)) \\
&\leq M \sum_{k=1}^\infty \exp(2(2^k + 1)+1) \exp(-2 (2^k + u^2)) \\
&\leq K M \exp(-2 u^2).
\end{align*}
For each element $w \in W$, we define a sequence of (random) integers $k(w,q)$ in the following inductive manner. First, let $k(w,0) = 0$, and for each $q \geq 1$, we define
\[
k(w,q) = \inf\left\lbrace k : k \geq k(w,q-1), \, d_\rho(w, U_k^W) \leq \frac12 d_\rho(w, U_{k(w,q-1)}^W) \right\rbrace.
\]
Now, consider elements $\pi_q(w) \in U_{k(w,q)}^W$ satisfying $d_\rho(w, \pi_q(w)) = d_\rho(w, U_{k(w,q)}^W)$. By induction, we find that $d_\rho(w, \pi_q(w)) \leq \rho 2^{-q}$. Furthermore, when $\Omega(u)$ occurs, since $\pi_q(w) \in U_{k(w,q)}^W$ and $\pi_{q-1}(w) \in U_{k(w,q-1)}^W \subset U_{k(w,q)}^W$, it follows that
\[
|Y_{\pi_q(w)} - Y_{\pi_{q-1}(w)}| \leq 2(2^{k(w,q)/2} + u) d_\rho(\pi_q(w), \pi_{q-1}(w)).
\]
Therefore, letting $w_0 \in T_0^W$, under $\Omega(u)$, 
\begin{align*}
|Y_w - Y_{w_0}| &\leq \sum_{q = 1}^{\infty} |Y_{\pi_q(w)} - Y_{\pi_{q-1}(w)}| \\
&\leq \sum_{q = 1}^{\infty} 2(2^{k(w, q)/2} + u) d_\rho(\pi_q(w), \pi_{q-1}(w)) \\
&\leq \sum_{q = 1}^{\infty} 2(2^{k(w, q)/2} + u) d_\rho(w, \pi_q(w)) + \sum_{q=1}^{\infty} 2(2^{k(w, q)/2} + u) d_\rho(w, \pi_{q-1}(w)).
\end{align*}
By construction,
\[
\sum_{q=1}^\infty 2^{k(w,q)/2} d_\rho(w, \pi_q(w)) \leq \sum_{k=0}^{\infty} 2^{k/2} d_\rho(w, T_n^W) \leq 2 \hat{\gamma}_2(W, d_\rho).
\]
Similarly, by the definition of $k(w,q)$, it follows that $d_\rho(w, U_{k(w,q)}^W) \leq \frac12 d_\rho(w, U_{k(w,q-1)}^W)$ and $d_\rho(w, U_{k(w,q)-1}^W) \geq \frac12 d_\rho(w, U_{k(w,q-1)}^W)$. Therefore,
\begin{eqnarray*}
\sum_{q=1}^\infty 2^{k(w,q)/2} d_\rho(w, \pi_{q-1}(w)) 
  &\leq& 2 \sum_{q=1}^\infty 2^{k(w,q)/2} d_\rho(w, T_{k(w,q) - 1}^W) \\
  &\leq& 4 \sum_{k=0}^\infty 2^{k/2} d_\rho(w, T_n^W) \\
  &\leq& 8 \hat{\gamma}_2(W, d_\rho).
\end{eqnarray*}
Finally, we have that $\sum_{q=1}^\infty d_\rho(\pi_q(w), w) \leq \rho \sum_{q=1}^\infty 2^{-q} = \rho$, and $\sum_{q=1}^\infty d_\rho(\pi_{q-1}(w), w) \leq 2 \rho$. Therefore, when $\Omega(u)$ occurs, for any $w \in W$,
\[
|Y_w - Y_{w_0}| \leq K(\hat{\gamma}_2(W, d_\rho) + \rho u) \leq K(\gamma_2^\rho(W) + \rho u).
\]
Altogether, this implies that
\[
\mathbb{P}\left(\sup_{w,w' \in W} |Y_w - Y_{w'}| > K(\gamma_2^\rho(W) + \rho u)\right) \leq K \exp(I_\infty(X, W)-2 u^2).
\]
Since $\mathbb{E}Y_w = 0$, $\sup_{w \in W} |Y_w| \leq \sup_{w, w' \in W} |Y_w - Y_{w'}|$, and (\ref{eq:GenBoundProb}) follows. We would now like to apply \citet[Lemma 1]{xu2017information} to show (\ref{eq:GenBoundExp}). To do so, it is necessary to show that $\sup_{w \in \tilde{W}} |\mathcal{R}_n(w) - \mathcal{R}(w)|$ is subgaussian, where $\tilde{W}$ is an independent copy of $W$. Recall that a random variable $X$ is $\sigma$-subgaussian if $\log\mathbb{E}\exp(\lambda (X-\mathbb{E}X)) \leq \lambda^2 \sigma^2 / 2$. First, consider the case where $\tilde{W} = \mathcal{W}$ is a deterministic set of weights, and for brevity, let $\bar{\mathcal{R}}_n(w) = \mathcal{R}_n(w) - \mathcal{R}(w)$. In this case, one may apply McDiarmid's inequality to $\sup_{w \in \mathcal{W}} |\bar{\mathcal{R}}_n(w)| = f(X_1,\dots,X_n)$, where
\[
f(x_1,\dots,x_n) = \sup_{w \in \mathcal{W}} \left|\frac1n \sum_{i=1}^n \ell(x_i, w) - \mathbb{E}\ell(X_i, w)\right|.
\]
Since $\ell$ is bounded, it follows that $|\ell(x_i,w) - \ell(y_i,w)| \leq 2 B$ for any $x,y$ and $w \in \mathcal{W}$. Therefore,
\[
|f(x_1,\dots,x_{i-1},x_i,x_{i+1},\dots,x_n) - f(x_1,\dots,x_{i-1},x_i',x_{i+1},\dots,x_n)| \leq \frac{2 B}{n}.
\]
Applying McDiarmid's inequality reveals that
\[
\mathbb{P}\left(\left|\sup_{w\in\mathcal{W}}|\bar{\mathcal{R}}_n(w)|-\mathbb{E}\sup_{w\in\mathcal{W}}|\bar{\mathcal{R}}_n(w)|\right|>u\right)\leq2\exp\left(-\frac{nu^{2}}{2B^{2}}\right).
\]
Since the bound does not depend on $\mathcal{W}$, and $\{X_i\}_{i=1}^n$, $\tilde{W}$ are independent, we can condition on $\tilde{W}$ and apply this bound to find that
\[
\mathbb{P}\left(\left|\sup_{w\in\tilde{W}}|\bar{\mathcal{R}}_n(w)|-\mathbb{E}\sup_{w\in\tilde{W}}|\bar{\mathcal{R}}_n(w)|\right|>u\right)\leq2\exp\left(-\frac{nu^{2}}{2B^{2}}\right),
\]
which, by \citet[Theorem 2.1]{boucheron2013concentration}, implies that $\sup_{w \in \tilde{W}} |\bar{\mathcal{R}}_n(w)|$ is $(4 B / \sqrt{n})$-subgaussian. Applying \citet[Lemma 1]{xu2017information},
\begin{equation}
\label{eq:ExpProofPart1}
\mathbb{E}\sup_{w \in W} |\mathcal{R}_n(w) - \mathcal{R}(w)| \leq \mathbb{E}\sup_{w \in \tilde{W}} |\mathcal{R}_n(w) - \mathcal{R}(w)| + \sqrt{\frac{32 B^2 I_1(X, W)}{n}}.
\end{equation}
Using (\ref{eq:ChainingSubg}), an application of \citet[Proposition 2.4]{talagrand1996majorizing} shows that
\begin{equation}
\label{eq:ExpProofPart2}
\mathbb{E}\sup_{w \in \tilde{W}} |Y_w| \leq \mathbb{E}\left[ \mathbb{E}_{\tilde{W}} \sup_{w \in \tilde{W}} |Y_w|\right] \leq K \mathbb{E} \gamma_2^\rho(\tilde{W}) = K \mathbb{E} \gamma_2^\rho(W),
\end{equation}
where $\mathbb{E}_{\tilde{W}}$ denotes conditional expectation, conditioned on $\tilde{W}$. The result now follows by combining (\ref{eq:ExpProofPart1}) and (\ref{eq:ExpProofPart2}). 
\end{proof}

\begin{remark}
Unfortunately, there has been little work on providing good estimates on the universal constant $K$. To our knowledge, only the original work of Fernique reports constants: from \citet[pg. 74]{fernique1975regularite}, we find that
\[
K \leq 30\sqrt{8} \left[\sqrt{2+\frac{1}{\log2}}+\frac{e}{2\sqrt{\log2}}\right] \approx 296,
\]
which is likely much larger than necessary. 
\end{remark}

\begin{proof}[Proof of Corollary \ref{cor:Hausdorff}]
Define a probability measure $\mu$ with support on $W$ by $\mu(E) = \mathcal{H}^{\alpha}(W \cap E) / \mathcal{H}^{\alpha}(W)$. By assumption, $\mu(B_r(w)) \geq (C_\rho r)^{\alpha}$ for $0 \leq r < \rho$ and any $w \in W$. Therefore,
\[
\gamma_2^\rho(W) \leq \sup_{w \in W} \frac1{\rho} \int_0^\rho \sqrt{\log \frac{1}{\mu(B_r(w))}} \dd r \leq (C_{\rho}^{-1} \sqrt{\alpha}) \cdot \frac{1}{\rho} \int_0^{\rho C_{\rho}} \sqrt{\log\frac{1}{r}} \dd r.
\]
Since $\mu(B_r(w)) \leq 1$, it follows that $\rho C_\rho \leq 1$. The result follows upon the observation that $\int_0^1 \sqrt{\log \frac1{r}} \dd r = \frac{\sqrt{\pi}}{2}$. 
\end{proof}

\subsection{Proof of Theorem \ref{thm:bound_transition_kernel}}

The Dudley bound is related to the Fernique--Talagrand functional through the following lemma, which combines \citet[Corollary 2.3.2]{talagrand2014upper} with the discussion on \citet[pg. 22]{talagrand2014upper}. Let $N_{r}^{d}(W)$ denote the $r$-covering number of $W$, that is, the smallest integer $N$ such that there exists a set of $N$ balls of radius $r$ under the metric $d$, whose union contains $W$. 

\begin{lemma}[Dudley entropy]
\label{lem:Dudley}
There exists a universal constant $K > 0$ such that for any metric $d$ and set $W$, $\gamma_2(W, d) \leq K \int_0^\infty \sqrt{\log N_r^d(W)} \dd r$. In particular, $\gamma^\rho_2(W) \leq \frac{K}{\rho}\int_0^\rho \sqrt{\log N_r(W)} \dd r$.
\end{lemma}

If $W \subset \mathbb{R}^D$, the Dudley bound is never off by any more than a factor of $\log(d+1)$ \citep[Exercise 2.3.4]{talagrand2014upper}. Now, since $x \mapsto \sqrt{\log x}$ is concave on $[1,\infty)$, Jensen's inequality implies that $\mathbb{E}\sqrt{\log X} \leq \sqrt{\mathbb{E} \log X}$ for any random variable with support in $[1,\infty)$. Therefore, 
\[
\mathbb{E}\gamma^\rho_2(W) \leq \frac{K}{\rho} \int_0^\rho \sqrt{\log \mathbb{E}N_r(W)} \dd r.
\]
Covering numbers for images of Markov processes are bounded by the following fundamental lemma. 
\begin{lemma}
\label{lem:Covering}
Let $\Lambda(r)$ be a fixed collection of cubes of side length $r$ in $\mathbb{R}^{d}$ such that no ball of radius $r$ can intersect more than $K$ cubes of $\Lambda(r)$. 
\begin{enumerate}[leftmargin=*]
\item Suppose that $X_{n}$ is a time-homogeneous Markov chain with $n$-step transition kernel $P^{n}(x,A)$. For any integer $m$, let $\mathcal{N}_{r}(m)$ denote the number of cubes in $\Lambda(r)$ hit by $X_{n}$ at some time $0\leq k\leq m$. Then
\[
\mathbb{E}\mathcal{N}_{r}(m)\leq2K\left[\inf_{x\in\bigcup_{r>0}\Lambda(r)}\mathbb{E}_{x}\left(\frac{1}{m}\sum_{k=1}^{m}P^k(x, B_{r/3}(x))\right)\right]^{-1}.
\]
\item Suppose that $X_{t}$ is a time-homogeneous strong Markov process in $\mathbb{R}^{d}$ with transition kernel $P(t,x,A)$. For any $t\geq0$, let $\mathcal{N}_{r}(t)$ denote the number of cubes in $\Lambda(r)$ hit by $X_{t}$ at some time $s\in[0,t]$. Then
\[
\mathbb{E}\mathcal{N}_{r}(t)\leq2K\left[\inf_{x\in\bigcup_{r>0}\Lambda(r)}\mathbb{E}_{x}\left(\frac{1}{t}\int_{0}^{t}P(s, x, B_{r/3}(x))\dd s\right)\right]^{-1}.
\]
\end{enumerate}
\end{lemma}
\begin{proof}
The second result is precisely \citet[Lemma 3.1]{liu1998hausdorff}, so it will suffice to show only the first. In a similar fashion, we consider a sequence of stopping times constructed in the following manner: let $\tau_{0}=0$ and for each positive integer $j$, we let \[
\tau_{j}=\min\left\{ k\geq\tau_{j-1}:\min_{i=0,\dots,j-1}|X_{k}-X_{\tau_{i}}|>r\right\}.
\] 
In other words, each $\tau_{j}$ is chosen to be the first time that the Markov chain is at least distance $r$ from $X_{\tau_{0}},\dots,X_{\tau_{j-1}}$. By construction, $|X_{\tau_{j}}-X_{\tau_{k}}|\geq r$ for $j\neq k$. By a Vitali covering argument, the balls $\{B_{r/3}(X_{\tau_{j}})\}_{j}$ are disjoint. Now, let
\[
T_{j}=\sum_{k=\tau_{j}+1}^{\tau_{j}+m}\ind_{X_{k}\in B_{r/3}(X_{\tau_{j}})}
\]
be the sojourn time of $X_{k}$ in $B_{r/3}(X_{\tau_{j}})$ in the interval $(\tau_{j},\tau_j + m]$. Furthermore, let $\eta=\min\{k:\tau_{k}>m\}$ so that $\{X_{k}\}_{k=0}^{m}\subset\bigcup_{j=0}^{\eta-1}B_{r}(X_{\tau_{j}})$.
Therefore, because $\{X_{k}\}_{k=0}^{m}$ is contained within the union of $\eta$ balls in $\mathbb{R}^{d}$ and no ball in $\mathbb{R}^{d}$ can intersect any more than $K$ cubes of $\Lambda(r)$, it follows that 
\begin{equation}
\label{eq:NumCubesBound}
\mathcal{N}_{r}(m)\leq K\eta. 
\end{equation}
Let $I_{j}$ be the indicator of the event $\{\tau_{j}\leq m\}$, or equivalently, $\{\eta-1\geq j\}$. Doing so, we have that $\eta=\sum_{j=0}^{\infty}I_{j}$. Furthermore, since $\sum_{j=0}^{\infty}\ind_{X_{k}\in B_{r/3}(X_{\tau_{j}})}\leq1$ by the disjointness of $\{B_{r/3}(X_{\tau_{j}})\}_{j}$,
\begin{align*}
\sum_{j=0}^{\infty} I_{j}T_{j}	&=\sum_{j=0}^{\infty} \sum_{k=\tau_{j}+1}^{\tau_{j}+m}\ind_{\tau_{j}\leq m}\ind_{X_{k}\in B_{r/3}(X_{\tau_{j}})}\\
&\leq\sum_{k=1}^{2m}\sum_{j=0}^{\infty}\ind_{X_{k}\in B_{r/3}(X_{\tau_{j}})}=2m.
\end{align*}	
By the strong Markov property, we may condition on starting the process at $X_{\tau_{j}}$:
\begin{align*}
\mathbb{E}[I_{j}T_{j}]	&=\mathbb{E}\left[\ind_{\tau_{j}\leq m}\mathbb{E}_{X_{\tau_{j}}}\sum_{k=1}^{m}\ind_{X_{k}\in B_{r/3}(X_{\tau_{j}})}\right]\\
&\geq\mathbb{E}I_{j}\cdot\inf_{x\in\mathbb{R}^{d}}\mathbb{E}_{x}\sum_{k=1}^{m}\ind_{X_{k}\in B_{r/3}(x)}.
\end{align*}
Therefore, by monotone convergence,
\begin{align*}
\mathbb{E}\eta\cdot\inf_{x\in\mathbb{R}^{d}}\mathbb{E}_{x}\sum_{k=1}^{m}\ind_{X_{k}\in B_{r/3}(x)}&=\sum_{j=0}^{\infty}\mathbb{E}I_{j}\inf_{x\in\mathbb{R}^{d}}\mathbb{E}_{x}\sum_{k=1}^{m}\ind_{X_{k}\in B_{r/3}(x)}\\
&\leq\sum_{j=0}^{\infty}\mathbb{E}[I_{j}T_{j}]=\mathbb{E}\sum_{j=0}^{\infty}I_{j}T_{j}\leq2m,
\end{align*}
and hence
\begin{equation}
\label{eq:NumCubesBound2}
\mathbb{E}\eta\leq2\left[\inf_{x\in\mathbb{R}^{d}}\mathbb{E}_{x}\left(\frac{1}{m}\sum_{k=1}^{m}\ind_{X_{k}\in B_{r/3}(x)}\right)\right]^{-1}.
\end{equation}
The result now follows from (\ref{eq:NumCubesBound}) and (\ref{eq:NumCubesBound2}). 
\end{proof}
If $\Lambda(r)$ is chosen to be the set of dyadic cubes in $\mathbb{R}^D$ of side length $r$, then $K = 3^D$. Combining Lemmas~\ref{lem:Dudley} and \ref{lem:Covering} with this choice of $\Lambda(r)$ yields Theorem \ref{thm:bound_transition_kernel}. Note that the dimension dependence arises only due to this particular choice of $\Lambda(r)$. For our purposes, $K$ is mostly irrelevant, but it is worth noting that this dimension dependence could feasibly be removed with a less naive choice of $\Lambda(r)$.

\subsection{Proof of (\ref{eq:SDE_FT})}
\label{sec:SDE_App}
The proof of (\ref{eq:SDE_FT}) relies on the following simple lemma.
\begin{lemma}
\label{lem:GaussInt}
For any $a > 0$, $x \in \mathbb{R}^D$, and $0 \leq r \leq \rho$,
\[
\frac{r^{D}}{2}I_\rho(a,D)\leq\int_{B_{r}(x)}e^{-a\|y-x\|^{2}}\dd y\leq\frac{r^{D}}{D},
\]
where $I_\rho(a,D) = \int_0^1 v^{D/2-1} e^{-av \rho^2} \dd v$ is monotone decreasing in $a$, $D$, and $\rho$.
\end{lemma}
\begin{proof}
By a change of variables, 
\[
\int_{B_r(0)} \exp(-a\|x\|^2) \dd x = \int_0^r u^{D-1} e^{-a u^2} \dd u = \frac{r^D}{2} \int_0^1 v^{D/2 - 1} e^{-a r^2 v} \dd v.
\]
The bounds are obtained through $e^{-a v \rho^2}\leq e^{-a r^2 v} \leq 1$.
\end{proof}
\begin{proof}[Proof of (\ref{eq:SDE_FT})]
In the sequel, $K$ will be used to denote a universal constant, not necessarily the same at each appearance. First, by the Aronson estimate and Lemma \ref{lem:GaussInt},
\begin{align*}
\int_0^T P_t(x,B_r(x)) \dd t &\geq K \int_0^T \int_{B_r(x)}  \frac{1}{t^{D/2}} \exp\left(-\frac{\|x-y\|^2}{2\lambda t}\right) \dd y \dd t \\
&\geq K r^D \int_0^T \frac{1}{t^{D/2}} I_\rho\left(\frac{1}{2\lambda t}, D\right) \dd t. 
\end{align*}
By Fubini's Theorem, and through the change of variables $t \mapsto t^{-1}$,
\begin{align*}
\int_0^T P_t(x,B_r(x)) \dd t &\geq K r^D \int_0^1 \int_0^T \frac{1}{t^{D/2}} v^{D/2 - 1} e^{-v\rho^2 / (2\lambda t)} \dd t  \dd v \\
&\geq K r^D \int_0^1 \int_{1/T}^\infty s^{D/2 - 2} v^{D/2 - 1} e^{-s v\rho^2 / 2\lambda t} \dd s  \dd v \\
&\geq K T r^D J_{\rho,T}\bigg(\frac{1}{2\lambda}, D\bigg),
\end{align*}
where $J_{\rho,T}$ is as defined in Section \ref{sec:Bounds}.
Equation (\ref{eq:SDE_FT}) now follows from Theorem \ref{thm:bound_transition_kernel}.
\end{proof}

\subsection{Proof of Corollary \ref{cor:BulkExponent}}
The proof of Corollary \ref{cor:BulkExponent} itself relies on the following corollary, which performs the local homogeneity approximation to $\{W_k\}_{k=0}^m$. 
\begin{corollary}
\label{cor:FTIID}
Let $\Omega$ be a closed set such that $\mathbb{P}(W_k \notin \Omega \text{ for some } k = 0,\dots,m) \leq \zeta_m$. Then for any probability measure $\mu$, letting $\mu^k$ denote $k$-fold convolution of $\mu$ and $P_\Omega$ the transition kernel of $W_k$ conditioned on $W_k \in \Omega$ for $k=1,2,\dots,m$,
\begin{multline}
\label{eq:FTIID}
\mathbb{E}\gamma_{2}^{\rho}(\{W_{k}\}_{k=0}^{m})\leq\frac{K}{\rho}\int_{0}^{\rho}\sqrt{(D+2)\log3-\log\left(\frac{1}{m}\sum_{k=1}^{m}\mu^{k}(B_{r})\right)} \dd r \\
+\sqrt{\log(m+1)}\left(\zeta_{m}+m\sup_{x\in\Omega}d_{\text{TV}}(P_{\Omega}(x,x+\cdot),\mu)\right).
\end{multline}
\end{corollary}

\begin{proof}
Let $\{\bar{W}_k\}_{k=0}^m$ denote the random walk satisfying $\bar{W}_{k+1} = \bar{W}_k + Z_k$, where each $Z_k \sim \mu$ is independent. Since $\gamma_2^\rho(W) \leq \sqrt{\log |W|}$, it follows that
\begin{align*}
|\mathbb{E}\gamma_2^\rho(\{W_k\}_{k=0}^m) - \mathbb{E}\gamma_2^\rho(\{\bar{W}_k\}_{k=0}^m)| &\leq \sqrt{\log(m+1)} \sup_{\|f\|_{\infty} \leq 1} |\mathbb{E}f(\{W_k\}_{k=0}^m) - \mathbb{E}f(\{\bar{W}_k\}_{k=0}^m)|\\
&\leq \sqrt{\log(m+1)} d_{\text{TV}} (\{W_k\}_{k=0}^m, \{\bar{W}_k\}_{k=0}^m).
\end{align*}
Since $\mathbb{E}\gamma_2^\rho(\{\bar{W}_k\}_{k=0}^m)$ is bounded above by the first term of (\ref{eq:FTIID}) by Theorem \ref{thm:bound_transition_kernel}, it suffices to show that
\[
d_{\text{TV}}(\{W_k\}_{k=0}^m, \{\bar{W}_k\}_{k=0}^m) \leq \zeta_m + m \sup_{x \in \Omega}d_{\text{TV}}(P_\Omega(x,x+\cdot),\mu).
\]
Let $Z_k$ be chosen such that $\{W_k\}_{k=0}^m$ and $\{\bar{W}_k\}_{k=0}^m$ are optimally coupled under the total variation metric, that is,
\[
d_{\text{TV}}(\{W_k\}_{k=0}^m, \{\bar{W}_k\}_{k=0}^m)
= \mathbb{P}(\{W_k\}_{k=0}^m \neq \{\bar{W}_k\}_{k=0}^m).
\]
Conditioning on $W_k \in \Omega$, there is
\begin{align*}
d_{\text{TV}}(\{W_k\}_{k=0}^m, \{\bar{W}_k\}_{k=0}^m)
&\leq \zeta_m + \mathbb{P}(\{W_k\}_{k=0}^m \neq \{\bar{W}_k\}_{k=0}^m \vert \{W_k\}_{k=0}^m \in \Omega^m)
\\&\leq \zeta_m + \sum_{j=1}^m\mathbb{P}(W_j \neq \bar{W}_j \vert \{W_k\}_{k=0}^m \in \Omega^m, \{W_k\}_{k=0}^{j-1} = \{\bar{W}_k\}_{k=0}^{j-1}),
\end{align*}
which implies (\ref{eq:FTIID}).
\end{proof}
\begin{proof}[Proof of Corollary \ref{cor:BulkExponent}]
By the hypotheses, for $\bar{c}_m = m^{-1}(c_1+\cdots+c_m)$, $m^{-1} \sum_{k=1}^m \mu^k(B_r) \geq \bar{c}_m r^{\alpha}$ for any $0 < r < r_0$. Let $\epsilon > 0$ be arbitrary. Note that if $r \leq \rho_\epsilon \coloneqq \min\{1,r_0,(\bar{c}_m / 3^{D+2})^{1/\epsilon}\}$, 
\[
\frac{1}{m} \sum_{k=1}^m \mu^k(B_r) \geq 3^{D+2} r^{\alpha + \epsilon}.
\]
Therefore, considering the first term on the right hand side of (\ref{eq:FTIID}), since $\rho_\epsilon \leq 1$,
\[
\int_{0}^{\rho_\epsilon}\sqrt{(D+2)\log3-\log\left(\frac{1}{m}\sum_{k=1}^{m}\mu^{k}(B_{r})\right)} \dd r \leq \sqrt{\alpha + \epsilon}\int_{0}^{1} \sqrt{\log\frac{1}{r}}\dd r.
\]
The result now follows from Corollary \ref{cor:FTIID}.
\end{proof}
}
\fi

\end{document}